%% file: iclr2027_conference.tex
\documentclass{article} 
\usepackage{iclr2027_conference,times}

\iclrfinalcopy
\input{math_commands.tex}

\usepackage{hyperref}
\usepackage{url}

\usepackage{amsmath}
\usepackage{enumitem}
\usepackage{amssymb}
\usepackage{amsthm}
\usepackage{algorithm}
\usepackage{algorithmic}
\usepackage{booktabs}
\usepackage{array}
\usepackage{multirow}
\usepackage{graphicx}
\usepackage{wrapfig}
\usepackage{xcolor}
\usepackage{tabularx}

\newtheorem{proposition}{Proposition}
\newtheorem{theorem}{Theorem}

\title{Large Distant Gradients Need Not Be Reliable: reliability-weighted credit assignment for long-horizon autoregressive forecasting}

\author{Junhao Zhao \\
University of Maryland, College Park \\
\texttt{jzhao121@umd.edu}
\And
David Michael Simberg \\
University of Maryland, College Park \\
\texttt{dsimberg@umd.edu}
\And
Jacob Kang \\
University of Maryland, College Park \\
\texttt{jkang115@umd.edu}
\And
Colin Connor Kurniawan \\
University of Maryland, College Park \\
\texttt{ckurniaw@umd.edu}
\And
Nan Xu\thanks{Corresponding author.} \\
University of Maryland, College Park \\
\texttt{nanxu@umd.edu}
}

\hypersetup{hidelinks}
\begin{document}

\addtocontents{toc}{\protect\setcounter{tocdepth}{-1}}

\maketitle
\fancyhead{}
\renewcommand{\headrulewidth}{0pt}
\begin{abstract}
In autoregressive forecasting, long prediction rollouts provide distant supervision, but backpropagation through time (BPTT) carries gradients from those losses through many autoregressive steps. Repeated Jacobian products can make distant gradients dominate the update while amplifying predictable signal and unpredictable innovation together; a large distant gradient therefore need not carry reliable learning signal. Motivated by this observation, we introduce Internal Dual-Wiener routing (Internal-DW), a principled backward-only intervention that preserves the full forward rollout and all horizon losses while reliability-weighting internal gradient routes. At each residual block, we derive bounded Wiener gains for the identity and nonlinear routes that balance preserving predictable learning signal against suppressing unpredictable variation, and estimate them from route-level gradient statistics and an explicit noise model. In a controlled system with known gradient signal-to-noise ratio (SNR), we show that distant gradients can grow even as their SNR falls, and that Internal-DW reduces error in recovering predictable gradient signals and improves forecasting. On four history-dominated, weak-drive testbeds, Internal-DW reduces forecast error by 5.2\%–13.8\% relative to full BPTT, outperforms gradient clipping and Jacobian regularization on three testbeds, with similar performance on shear flow, and outperforms validation-selected truncated BPTT (TBPTT) on three. It also extends or preserves the fitted optimal training-horizon range across these four testbeds. Across the full benchmark suite, the current Internal-DW estimator has a clear applicability boundary: its benefit diminishes or reverses when usable history is limited or when the selected sampler fails to represent dominant drive-dependent variation. The results show that retaining long-horizon supervision does not require trusting every backward contribution equally. The code is available at \url{https://github.com/inspirelab-site/Internal-DW}.
\end{abstract}

\section{Introduction}
In autoregressive forecasting, training over long closed-loop rollouts provides
supervision at distant forecast horizons, while backpropagation through time
(BPTT) carries gradients from those distant losses backward through many earlier
autoregressive steps. Classical analyses show that these distant gradients can grow
rapidly in recurrent models
~\citep{bengio1994learning,pascanu2013difficulty}, and related sensitivity
amplification has been observed in neural PDE solvers and weather models
~\citep{mccabe2023towards,pervez2026controlling}.  
But a large gradient need not be a reliable one~\citep{parmas2018pipps}. We ask when distant
BPTT gradients acquire high leverage without acquiring correspondingly reliable
learning signal, and how long-horizon supervision can be retained without
trusting every backward contribution equally. 

To see why gradient magnitude can be misleading, first consider the history-dominated case: an autonomous system without external forcing or
control inputs. The loss at each future step produces a gradient that is propagated backward through earlier applications of the same model, and the contributions from all forecast steps are added to form the parameter update.  The same Jacobian products that amplify a contribution's predictable learning signal also amplify its unpredictable innovation. Consequently, a distant contribution can become large enough to dominate the summed update even when most of its magnitude reflects unreliable variation. Gradient magnitude therefore
measures backward sensitivity, not necessarily useful credit.

\begin{figure}[t]
  \vspace*{-11.5pt}
  \centering
  \setlength{\abovecaptionskip}{0pt}
  \setlength{\belowcaptionskip}{0pt}
  \includegraphics[
    width=\linewidth,
    trim=0 3pt 0 7pt
  ]{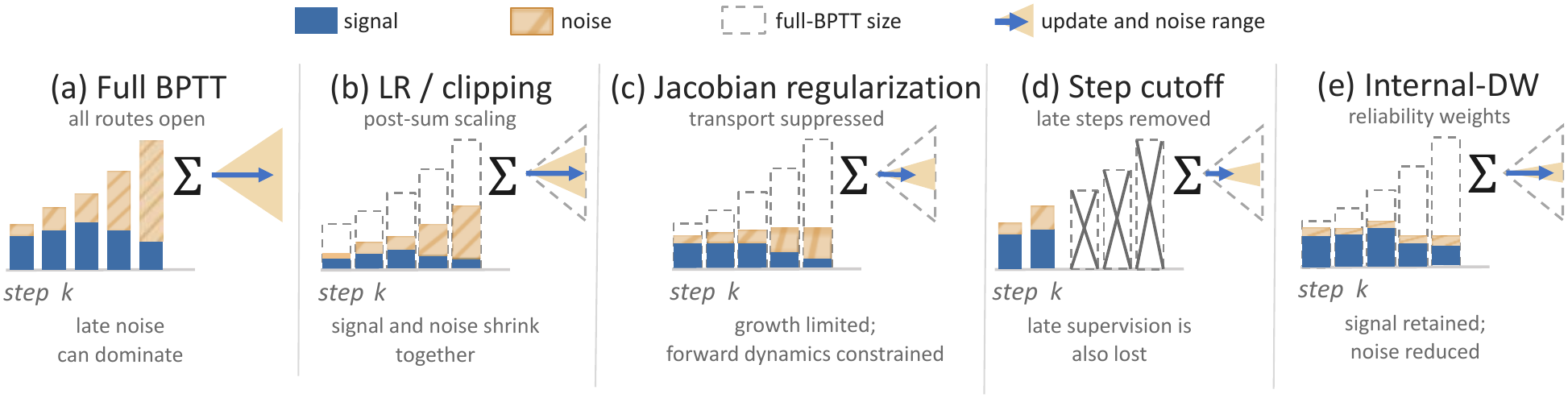}
  \vspace{-10pt}
\caption{\textbf{Common controls reshape long-horizon gradients in different ways.}
Conceptual comparison of forecast-step contributions retained by each
approach and resulting update. A route weight scales signal
and innovation together; these components are not separately observed in
data.}
  \label{fig:intro-gradient-controls}
  \vspace*{-12pt}
\end{figure}
Under full BPTT, all delayed contributions remain open and are summed into
the update (Fig.~\ref{fig:intro-gradient-controls}a). Existing controls address neighboring problems but intervene at different stages of the training process. Gradient clipping~\citep{pascanu2013difficulty} and adaptive optimizers such as Adam~\citep{kingma2015adam} rescale gradients after forecast-step contributions have already been combined and therefore cannot identify which delayed contribution is
unreliable (Fig.~\ref{fig:intro-gradient-controls}b). Forward Jacobian regularization instead limits amplification by constraining
the learned dynamics ~\citep{pervez2026controlling,nie2026jaws}
(Fig.~\ref{fig:intro-gradient-controls}c), whereas shorter rollouts remove distant supervision~\citep{temporal2025horizons}
(Fig.~\ref{fig:intro-gradient-controls}d), and TBPTT cuts long-range
gradient paths. None of these approaches directly asks which delayed contributions remain
reliable \emph{before} they are aggregated. This motivates a backward-only
intervention that preserves the complete rollout, all horizon losses, and the
learned forward dynamics while selectively reducing unreliable delayed credit (Fig.~\ref{fig:intro-gradient-controls}e).

Internal-DW implements reliability weighting inside residual blocks. At a block $y=x+F_\theta(x)$, BPTT maps an incoming gradient through
$I+J_F^\top$; Internal-DW replaces this backward operator with $\alpha I+mJ_F^\top$, assigning separate reliability weights to the identity and nonlinear routes while leaving the block's forward value unchanged. The two gains are chosen by a box-constrained $2\!\times\!2$ Wiener solve that trades distortion of the conditional gradient signal against unpredictable variation. Because signal and noise are not separately observed, the solve requires an explicit noise estimate~\citep{mehra1970identification,belanger1974noise,odelson2006autocovariance}. We therefore make the identifying assumption through two noise samplers: DW-Generic uses centered training residuals as a proxy for unpredictable error, whereas DW-Prior incorporates domain-supported structure when available.


We first test the mechanism in a controlled system where gradient signal-to-noise ratio (SNR) is known. Delayed gradients grow even as their SNR falls, the estimated gains reduce {route-gradient risk}, and Internal-DW achieves the lowest
forecast error among the tested methods. On application data, where the true signal--innovation split is unavailable, held-out gradient tests ask whether large per-horizon gradients transfer across examples, while controlled-noise tests ask whether the estimated gains decrease as noise increases. Across four history-dominated, weak-drive testbeds, Internal-DW improves forecasting and extends or preserves the useful training-horizon range. Its benefit weakens when usable history is limited or when important driven information is not represented by the sampler. These cases define boundaries of the current estimator, not of the broader delayed-credit reliability problem.

We therefore view unreliable delayed credit as the broader training problem,
and Internal-DW as one reliability-weighted correction for it. The current estimator is best suited when useful credit is primarily propagated through state history; separating that credit from newly injected drive remains an important extension.
\section{Related Work} \label{sec:related}
Existing approaches intervene at different stages of long-horizon training.
Scheduled sampling, DAgger, Professor Forcing, and pushforward training alter
the rollout distribution
\citep{bengio2015scheduled,ross2011reduction,lamb2016professor,
brandstetter2022message}, while Jacobian penalties and recurrent stabilization
constrain the learned forward dynamics
\citep{mccabe2023towards,pervez2026controlling,nie2026jaws,ye2025recurrent}.
Recent methods instead modify temporal supervision or credit assignment:
STLW gradually increases the weight of later losses \citep{coker2026scheduled}, HERO trains against a failed-rollout reference \citep{zhang2026hero}, and COLA uses dynamical criticality for online local credit assignment approximating global error propagation \citep{wang2026global}. TBPTT and stop-gradient methods limit delayed backpropagation, while damped BPTT imposes prescribed decay
\citep{tallec2017unbiasing,list2024differentiability,ingraham2019learning}.

Internal-DW differs in leaving the forward rollout, horizon losses, and learned
dynamics unchanged while estimating the reliability of transported internal
VJPs before their route contributions are aggregated. Broader connections to
reinforcement-learning credit assignment, differentiable world models,
recurrent-gradient estimators, covariance identification, and backward-message
methods are discussed in Appendix~\ref{app:related}.

\section{Reliable Gradient Transport in Long-Horizon Rollout}
\label{sec:collapse-theory}
\textbf{Our theoretical framework has three components.}
(1) We formalize the mismatch between delayed-gradient leverage and reliability.
(2) Theorem~1 derives Internal-DW as the minimum-risk bounded linear weighting of residual backward routes, with the forward rollout unchanged.
(3) We show that the optimal weights are not identifiable from open-route gradients alone, requiring an explicit noise model whose assumptions determine the estimator's scope.

\subsection{Large Delayed Gradients Need Not Be Reliable}
\label{sec:theory-estimand}

Jacobian transport controls a delayed gradient's magnitude, not its
reliability, and one amplified contribution can dominate the summed update.
For the $K$-step loss
$\mathcal L=K^{-1}\sum_{k=1}^{K}\lVert r_{t+k}\rVert^2$, where
$r_{t+k}=\hat x_{t+k}-x^\star_{t+k}$, let
$D_{t+k}=d\hat x_{t+k}/d\theta$. Define
$J_\tau:=\partial\hat x_{\tau+1}/\partial\hat x_\tau$,
$G_\tau:=\partial\hat x_{\tau+1}/\partial\theta$, and
$\Phi_{a,q}=J_{a+q-1}\cdots J_a$, with $\Phi_{a,0}=I$. The exact gradient and
its chain-rule expansion are
\begin{equation}
\nabla_\theta\mathcal L
=
\frac{2}{K}\sum_{k=1}^{K}D_{t+k}^\top r_{t+k},
\qquad
D_{t+k}^\top r_{t+k}
=
\sum_{j=0}^{k-1}
G_{t+j}^\top
\Phi_{t+j+1,k-j-1}^\top r_{t+k}.
\label{eq:horizon-gradient-sum}
\end{equation}
Because the same parameters are reused at every rollout step, each step injects
a new local contribution while transporting earlier ones through Jacobian
products. Appendix~\ref{sec:shared-gradient-accumulation} formalizes this
forced sensitivity recurrence and shows that, under a stationary local
linearization, sensitivity can plateau, accumulate, or grow exponentially.
Equation~\ref{eq:horizon-gradient-sum} sums contributions from all such paths,
so a dominant transported term can control the update's magnitude and direction
~\citep{mikhaeil2022difficulty}.

Write $r_{t+k}=b_{t+k}+\eta_{t+k}$, where $b_{t+k}=\E[r_{t+k}\mid\hat x_{t+k-1}]$ is the learnable conditional-mean component and $\eta_{t+k}$ is the unpredictable innovation under this conditioning. Here, ``unpredictable'' is relative to the declared conditioning information and does not imply physically irreducible noise. The same Jacobian product transports both components, so amplification increases the leverage of innovation without making it informative. Even without exponential growth, innovation comparable to the predictable component can perturb the update. We therefore seek a backward-only linear route weighting that minimizes
mean-squared error relative to the predictable delayed contribution, before
rollout-step contributions are aggregated.

\subsection{Internal-DW: Optimal Reliability-Weighted Gradient Transport}
\label{sec:method}

Internal-DW changes only backward propagation at residual merges; the forward
rollout is unchanged. We first define its residual-route operator and then show
that, under an explicit signal--innovation decomposition, its oracle weights
are the unique minimum-risk bounded linear weighting of the two backward
routes. Fig.~\ref{fig:internal-dw-method} summarizes the resulting method.

\paragraph{Backward-only routing.}
\label{sec:dual-wiener}
For a residual block $y=x+F_\theta(x)$, ordinary BPTT maps an incoming
gradient $v$ to $v+J_F^\top v$. Internal-DW instead uses
\begin{equation}
\operatorname{VJP}_{\mathrm{DW}}(v)
=
\alpha v+mJ_F^\top v,
\qquad
J_F=\frac{\partial F_\theta(x)}{\partial x}.
\label{eq:dual-route-jacobian}
\end{equation}
We assign one pair
$w_{k,\ell}=(\alpha_{k,\ell},m_{k,\ell})^\top$ to each horizon $k$ and
residual layer $\ell$.

\vspace{-7pt}
\begin{figure}[H]
\centering
\setlength{\abovecaptionskip}{2pt}
\includegraphics[width=\linewidth]
{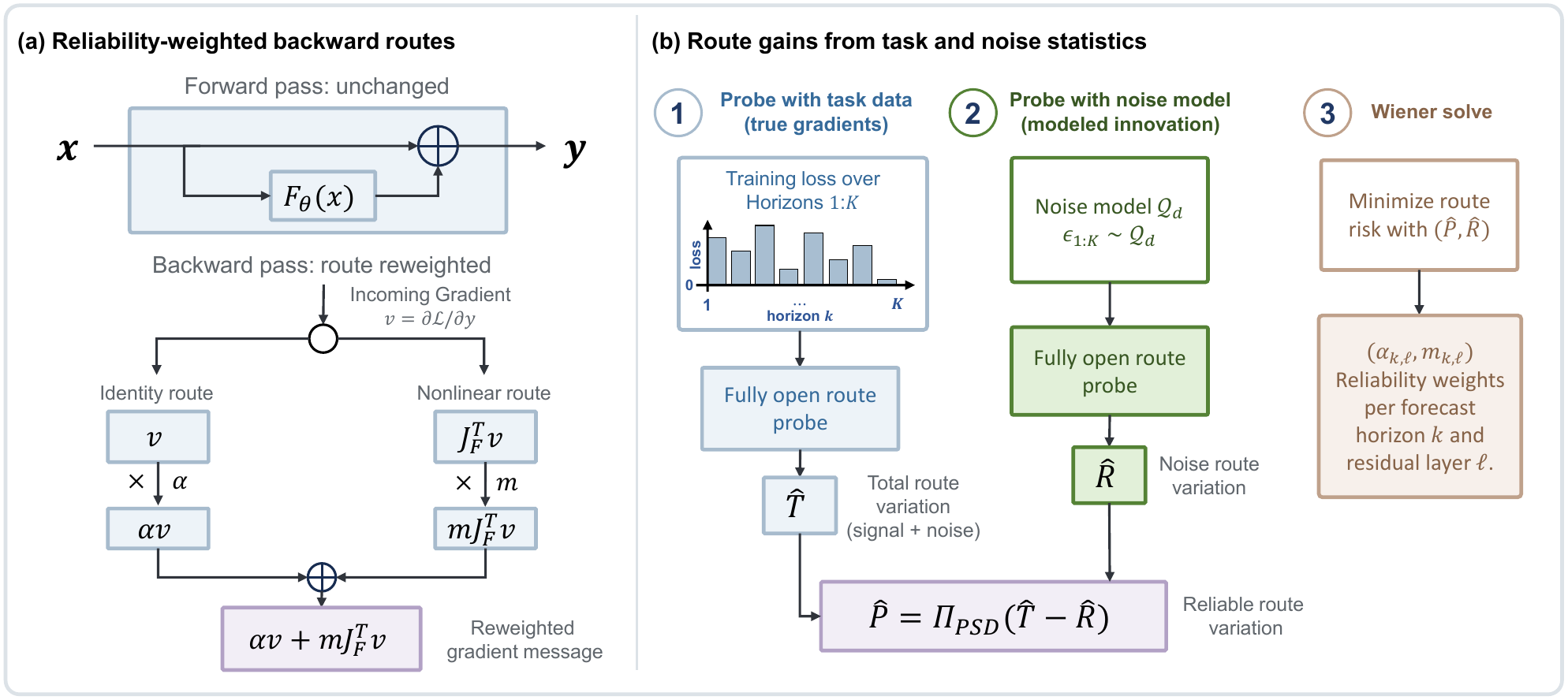}
\vspace{-10pt}
\caption[Internal-DW.]{\textbf{Internal-DW.}
(a) Internal-DW keeps the forward residual block unchanged and
reweights its two backward routes---identity and nonlinear---by
$(\alpha,m)$.
(b) The weights are calibrated per horizon $k$ and layer $\ell$.
Fully open probes produce total route variation $\widehat T$ from the
quadratic total-moment probe and noise route variation $\widehat R$ from a
noise model $\mathcal Q_d$. The predictable component
$\widehat P=\Pi_{\rm PSD}(\widehat T-\widehat R)$ and $\widehat R$
are substituted into the minimum-risk routing objective, yielding reliability
weights $(\alpha_{k,\ell},m_{k,\ell})$.}
\label{fig:internal-dw-method}
\vspace{-12pt}
\end{figure}

\paragraph{Wiener gains.}
At a fixed $(k,\ell)$, let $v_{k,\ell}$
be the gradient arriving at the output of the residual block. The identity
and nonlinear backward routes return
\[
d_I=v_{k,\ell},
\qquad
d_F=J_{F_\ell}^{\top}v_{k,\ell}.
\]
Collect these two route gradients as
$d=[\,d_I\;\;d_F\,]$ and decompose them as
\[
d=s+n,
\qquad
s=\E[d\mid\mathcal C],
\qquad
n=d-s,
\]
where $\mathcal C$ denotes the information under which delayed credit is
predictable. Thus $s$ is the predictable transported gradient and $n$ its
unpredictable innovation. Then $\E[s^\top n]=0$, and
\begin{equation}
T=\E[d^\top d]=P+R,
\qquad
P=\E[s^\top s],
\qquad
R=\E[n^\top n].
\label{eq:route-moment-decomposition}
\end{equation}

\begin{theorem}[Optimal reliability-weighted gradient transport]
\label{thm:optimal-routing}
For $w=(\alpha,m)^\top$, let $dw$ be the routed message and
$s\mathbf 1=s_I+s_F$ the fully open predictable target. If
$T=P+R\succ0$, then the route risk decomposes as
\begin{equation}
\mathcal R(w)
:=
\E\!\left[\left\|dw-s\mathbf 1\right\|_2^2\right]
=
(w-\mathbf 1)^\top P(w-\mathbf 1)+w^\top Rw .
\label{eq:dual-wiener-risk}
\end{equation}
The corresponding unconstrained and bounded optima are, respectively,
\begin{equation}
w_{\rm unc}=T^{-1}P\mathbf 1,
\qquad
w^\star=\arg\min_{w\in[0,1]^2}\mathcal R(w).
\label{eq:oracle-internal-dw-optimum}
\end{equation}
Hence $w^\star$ is the $T$-metric projection of $w_{\rm unc}$ onto
$[0,1]^2$. Since the fully open BPTT operator $w=\mathbf 1$ is feasible,
\begin{equation}
\mathcal R(w^\star)
\le
\mathcal R(\mathbf 1)
=
\mathbf 1^\top R\mathbf 1,
\label{eq:internal-dw-vs-bptt-risk}
\end{equation}
with strict inequality whenever $\mathbf 1$ is not itself optimal.
\end{theorem}

The first term in Equation~\ref{eq:dual-wiener-risk} measures distortion
of predictable credit from closing the routes, while the second measures
innovation retained by leaving them open. Thus Internal-DW is the
minimum-risk bounded linear trade-off between preserving predictable credit
and suppressing transported innovation, rather than a prescribed gradient
decay. For a single route, the unconstrained solution reduces to
$P/(P+R)$~\citep{wiener1949extrapolation}. The box constraint prevents sign
reversal and route-wise amplification. Appendix~\ref{app:optimal-routing-proof}
gives the proof and Appendix~\ref{app:route-weight-solve} gives the exact
constrained solve.

\paragraph{Plug-in estimation.}
The oracle moments $P$ and $R$ in Theorem~\ref{thm:optimal-routing} are not
separately observed. Fully open total-moment and noise probes produce route
messages $d^{\rm tot}$ and $d^{\rm noise}$, with the latter obtained by
transporting a sample from $\mathcal Q_d$. Suppressing $(k,\ell)$, we estimate
\begin{equation}
\widehat T
=
\operatorname{EMA}\!\left[(d^{\rm tot})^\top d^{\rm tot}\right],
\qquad
\widehat R
=
\operatorname{EMA}\!\left[(d^{\rm noise})^\top d^{\rm noise}\right],
\qquad
\widehat P
=
\Pi_{\rm PSD}(\widehat T-\widehat R).
\label{eq:plugin-route-moments}
\end{equation}
The total-moment probe uses the quadratic output loss specified in
Appendix~\ref{app:route-moment-calibration}, independently of the task loss
used for parameter updates. Substituting $(\widehat P,\widehat R)$ into
Equation~\ref{eq:dual-wiener-risk} and solving over $w\in[0,1]^2$ gives the
applied Internal-DW gains. Section~\ref{sec:gain-estimation} states the
assumptions on $\mathcal Q_d$; Appendix
Sections~\ref{app:routing-placement}--\ref{app:calibration-cost} give the
probes, implementation, and constrained solve; and Appendix
Sections~\ref{app:generic-sampler}--\ref{app:sampler-interface} give the
sampler constructions.

\subsection{Noise Identifiability and Sampler Assumptions}
\label{sec:gain-estimation}
Theorem~\ref{thm:optimal-routing} determines the oracle Internal-DW route
weights $w^\star$ from the latent decomposition $T=P+R$. However, the open
route law does not identify this decomposition by itself.

\paragraph{Route-law non-identifiability.}
Even the complete law of the open route vector does not, in general, identify
the split $T=P+R$, and hence does not identify the oracle routing weights.
To see the ambiguity, for any $0\preceq R\preceq T$, independent variables
$s\sim\mathcal N(0,T-R)$ and $n\sim\mathcal N(0,R)$ produce the same observed
distribution $z=s+n\sim\mathcal N(0,T)$, while different choices of $R$ can
yield different unconstrained optima $w_{\rm unc}=T^{-1}(T-R)\mathbf 1$. This non-identifiability is structural rather than a finite-sample limitation; Appendix~\ref{app:noise-nonidentifiability} proves the result formally. Internal-DW supplies the required identifying structure through $\mathcal Q_d$, which generates the output-noise trajectories used to estimate $R$.


\paragraph{DW-Generic.} DW-Generic uses centered, lagged training-residual trajectories, with a random trajectory-wide sign, as proxy samples of unpredictable innovation. Its assumption is correspondingly strong: centered prediction residuals must be a useful proxy for unpredictable innovation. If they contain systematic error that the forecasting model could still learn, DW-Generic incorrectly treats that
variation as noise.
\vspace{-0.2cm}
\paragraph{DW-Prior.} DW-Prior uses either cross-fitted residuals from a domain-supported predictor or samples from a domain-supported covariance model. The predictor-based version treats held-out variation unexplained by the predictor as noise and is appropriate when the predictor captures the conditional mean under the
declared conditioning information. The covariance-based version treats
zero-mean draws from the fitted covariance model as noise and is appropriate
when the selected covariance family captures the relevant dependence. For the
spatial samplers used here, this requires approximate translation stationarity
and a spectrum that remains stable during calibration.

Exact sampler constructions and implementation details are given in Appendix
Sections~\ref{app:generic-sampler}--\ref{app:sampler-interface}.
Appendix Table~\ref{tab:identification-summary} summarizes the
dataset-specific assumptions, supporting references, and retained covariance.
For each dataset, we select between DW-Generic and DW-Prior using validation
loss and then fix that choice for test evaluation.

Let $R_0$ denote the oracle noise moment, $R_Q$ the population noise moment
targeted by the selected sampler, and $\widehat R$ its finite-sample estimate.
Their difference decomposes as
\begin{equation}
\widehat R-R_0
=
\underbrace{(R_Q-R_0)}_{\text{sampler-model mismatch}}
+
\underbrace{(\widehat R-R_Q)}_{\text{finite-sample estimation error}}.
\label{eq:noise-moment-error}
\end{equation}
Thus data can reduce the second term in Equation~\ref{eq:noise-moment-error}, but cannot remove the first when the sampler model is misspecified. In particular, DW-Generic can count predictable model error as noise, while DW-Prior can omit predictable structure or covariance outside its chosen representation. The next subsection explains how these two error sources determine the scope of the current estimator.

\subsection{Scope and Boundaries of Internal-DW}
\label{sec:theory-predictability}
A first scope boundary arises when state history supplies little reliable delayed signal and no other signal source compensates. In this case, the
conditional signal moment $P$ is small relative to the oracle noise moment $R_0$, so the Wiener solution assigns little weight to the corresponding routes. Internal-DW then has little reliable delayed credit to preserve.

A boundary arises from sampler-model mismatch. Useful drive-dependent information should contribute to $P$ and favor a more open oracle route. If $\mathcal Q_d$ does not account for propagated history, observed drive, and their interaction, however, predictable driven variation can instead inflate $R_Q$ relative to $R_0$ along route directions. The applied gain may then close too far and suppress useful gradients. Even under a correctly specified sampler, finite data contribute the estimation error $\widehat R-R_Q$. Appendix~\ref{app:sampler-consistency} gives consistency conditions for the residual-based estimator, while Appendix~\ref{app:gain-error} shows how route-moment errors perturb the Wiener gain and contribute to excess local route risk. The first case is a low-signal scope boundary; sampler-model mismatch and finite-sample error are boundaries of the current estimator, not of the broader delayed-credit reliability problem.

\section{Experimental Setup}
\label{sec:experimental-setup}
\paragraph{Testbeds.} We use one controlled system with known SNR to examine the mechanism and eight application testbeds to evaluate forecasting performance. The controlled system is a stationary, undriven linear-Gaussian AR process. The application testbeds are Mackey--Glass (MG) for autonomous delayed dynamics \citep{mackey1977oscillation}; NARMA for driven nonlinear control \citep{atiya2000new}; iEEG and movie fMRI for neural recordings \citep{keles2024multimodal,vanessen2013hcp}; ETTm1 and ETTm2 for temporal forecasting \citep{zhou2021informer}; and The Well shear flow and WeatherBench-2 for gridded fluid and weather fields \citep{ohana2024well,rasp2024weatherbench2}. {Appendix~\ref{app:known-snr-signal-noise} details the controlled system; Appendix Sections~\ref{app:datasets}--\ref{app:training-details} cover the application datasets and training details, respectively.}

\paragraph{Predictive regimes.}
\label{sec:drive-history-boundary}
\begin{wrapfigure}{RH}{0.4\textwidth}
    \centering
        \vspace{-1cm}
    \setlength{\abovecaptionskip}{0pt}
    \makebox[\linewidth][c]{%
        \includegraphics[width=1.06\linewidth]{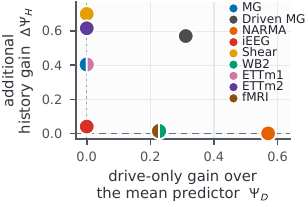}}
    \vspace{-11pt}
    \caption{\textbf{Predictive regimes.} Probes quantify observed-drive gain over mean predictor and additional history gain. MG+drive provides a controlled intervention on observed-drive gain.}
    \label{fig:drive-history-boundary-map}
    \vspace{-15pt}
\end{wrapfigure}

Fig.~\ref{fig:drive-history-boundary-map} maps predictive gains from simple held-out probes using observed future drive alone versus drive plus aligned history. Appendix Sections~\ref{app:regime-scores}--\ref{app:probe-validation} detail their construction, interpretation, and validation under controlled drive. The probes identify three predictive regimes and one controlled intervention:
\begin{itemize}[
leftmargin=1.15em,
labelsep=0.35em,
itemsep=0pt,
parsep=0pt,
topsep=2pt,
partopsep=0pt
]

\item \textbf{History-dominated, weak drive:}
MG, shear flow, ETTm1, and ETTm2.

\item
\textbf{Drive-dominated, weak history:} {NARMA, movie fMRI, and WeatherBench-2.}
\item
\textbf{Weak drive and weak history:}
{iEEG.}
\item
\textbf{{Joint history-and-drive intervention:}}
driven MG {(see Appendix~\ref{app:driven-mg-boundary}
for the driven-system configuration and intervention results).}

\end{itemize}
These groupings describe predictive structure recovered by the probes, not
intrinsic dataset properties, and predict the clearest Internal-DW benefit in
the history-dominated, weak-drive regime.

\paragraph{Evaluation and model selection.}
Estimator variants and checkpoints are selected by validation loss and fixed
before test evaluation. For a model trained with horizon \(K\), each held-out
origin is rolled out through \(H_{\rm eval}=\lceil1.5K\rceil\). For test units
\(\mathcal U\) and origins \(\mathcal O_u\), the primary metric is
\vspace{-5mm}

\begin{equation}
 \mathrm{RelL2}_{1:H_{\rm eval}}
 =\frac{1}{|\mathcal U|}\sum_{u\in\mathcal U}
   \frac{1}{|\mathcal O_u|}\sum_{o\in\mathcal O_u}
   \frac{1}{H_{\rm eval}}\sum_{h=1}^{H_{\rm eval}}
   \frac{\lVert\hat x_{u,o}(h)-x^\star_{u,o}(h)\rVert_2}
        {\max\{\lVert x^\star_{u,o}(h)\rVert_2,10^{-8}\}}.
 \label{eq:dense-multistart-rel-l2}
\end{equation}

Appendix~\ref{app:evaluation-protocol} gives the origin-selection details. 

\paragraph{Baselines and controls.} We compare Internal-DW with full BPTT, global gradient clipping (Clip), and forward-Jacobian regularization (JReg). On the four history-dominated, weak-drive testbeds, we additionally compare with TBPTT and a tied static gain
\(\alpha=m=c\) (Static), analogous to fixed backward
damping~\citep{ingraham2019learning}. The TBPTT segment length, Static coefficient \(c\), and choice between DW-Generic and DW-Prior are selected by validation loss and then fixed for test evaluation. Within each dataset, all methods otherwise use the same experimental settings.
Appendix Table~\ref{tab:identification-summary} lists the selected sampler, and Appendix~\ref{app:impl-details} gives architectures, selection grids, and training details.

\section{Results}
\label{sec:results}
The experiments first test whether delayed gradients become less reliable and
whether Internal-DW responds appropriately, then determine when this local
correction improves forecasting across predictive regimes and training
horizons.

\subsection{Mechanism and Estimator Response}
\label{sec:result-mechanism}
\paragraph{Known-SNR mechanism closure.}
In the known-SNR system, independently resampling future innovations from a fixed $x_t$ gives the exact conditional-mean gradient and therefore direct measurements of gradient SNR and prefix risk. A decrease in prefix risk means that an added forecast step contributes more conditional-mean signal than innovation; an increase means the reverse. Appendix Sections~\ref{app:known-snr-signal-noise}--\ref{app:known-snr-prefix-risk} give the construction. Fig.~\ref{fig:known-snr-failure} follows the resulting chain from gradient amplification to forecasting performance.

\vspace{-3pt}
\begin{figure*}[h]
  \centering
  \setlength{\abovecaptionskip}{2pt}
  \includegraphics[width=\textwidth]{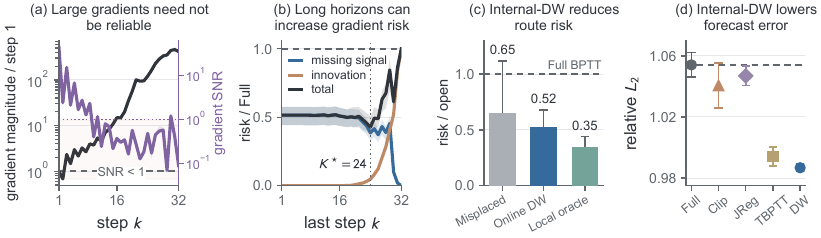}
    \vspace{-12pt}
  \caption{
  \textbf{Known-SNR mechanism closure.}
  (a) Per-horizon gradient magnitude grows while reliability falls;
  (b) Prefix gradient risk, decomposed into remaining predictable gradient signal and accumulated innovation noise;
  (c) Reliability weighting lowers {route risk}; and
  (d) Internal-DW gives the lowest forecast error. Appendix Sections~\ref{app:known-snr-signal-noise}--\ref{app:known-snr-gain-risk} give the complete protocol and values.}
  \label{fig:known-snr-failure}
\end{figure*}

Fig.~\ref{fig:known-snr-failure}a--b shows that later gradients grow while
their SNR falls. Beyond $K_{\rm grad}^\star=24$, accumulated innovation
outweighs the recovered signal and total prefix risk rises, showing that neither
gradient magnitude nor forecast distance reliably identifies useful credit.
Internal-DW reduces normalized {route risk} from $1.00$ for full BPTT
to $0.52$; shifting the same gains across forecast steps increases risk to
$0.65$, showing that their placement matters beyond generic attenuation
(Fig.~\ref{fig:known-snr-failure}c). {With access to the true signal moment and noise samples from the underlying process, the oracle selects better weights and further reduces risk to $0.35$, suggesting room to improve DW's gain estimation.} Internal-DW also achieves the lowest forecasting error among
the tested methods (Fig.~\ref{fig:known-snr-failure}d).

\paragraph{Application-data diagnostics.}
On application data, the true gradient SNR and oracle gains are unavailable, so we use two complementary held-out diagnostics. For each forecast step $k$, $A(k)$ measures the magnitude of that step's gradient relative to the first forecast step, while $U(k)$ measures whether its direction is useful on a disjoint held-out rollout. Positive $U(k)$ indicates that a small update along the negative step-$k$ training gradient would locally reduce held-out rollout loss, whereas $U(k)<0$ indicates that it would increase held-out loss. Appendix~\ref{app:heldout-gradient-utility} gives the exact definitions and evaluation protocol. Fig.~\ref{fig:application-data-diagnostics}a therefore asks whether gradients that become large at distant forecast steps also remain useful on unseen examples, while panel (b) asks whether the estimated gains close as independent noise increases.

\begin{figure*}[h]
    \centering
    \setlength{\abovecaptionskip}{2pt}
    \includegraphics[width=\textwidth]
    {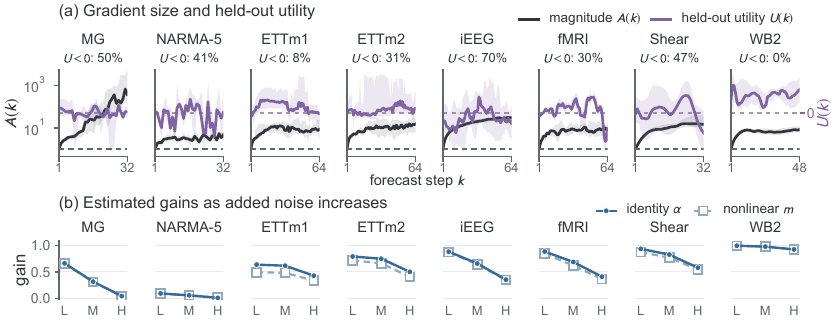}
    
\caption{\textbf{Application-data reliability and estimator response.} (a) Per-horizon relative gradient magnitude $A(k)$ and held-out utility $U(k)$. Solid curves show medians; shaded bands show the 10--90\% range for $A(k)$ and the interquartile range for $U(k)$. Percentages give the fraction of forecast steps with negative median $U(k)${, rounded to the nearest percent}.
(b) Mean identity and nonlinear gains across added-noise levels.
Within each dataset, L, M, and H correspond to residual-to-added-noise ratios
$4$, $1$, and $0.25$, respectively; smaller ratios indicate stronger added
noise. Both gains decrease as noise increases, showing the intended
directional response of the estimator.}
\vspace{-6pt}
    \label{fig:application-data-diagnostics}
\end{figure*}

Fig.~\ref{fig:application-data-diagnostics}a shows that gradient amplification and held-out utility can decouple on application data. Across most testbeds, per-horizon gradients grow even as their held-out utility fluctuates or becomes negative; ETTm1 and WeatherBench-2 provide complementary cases in which growing gradients remain mostly useful. This contrast shows that neither gradient magnitude nor forecast distance alone identifies a transferable delayed contribution. Fig.~\ref{fig:application-data-diagnostics}b supports the corresponding estimator-response claim: both route gains decrease on every testbed as controlled noise increases. The dataset-specific operating points remain informative. WeatherBench-2 stays nearly open, whereas NARMA is nearly closed even under weak perturbation, consistent with its selected sampler treating predictable drive-dependent residual variation as noise. Together, Fig.~\ref{fig:application-data-diagnostics} provides application-level support for two linked claims: large delayed gradients need not remain useful on held-out data, and Internal-DW closes its routes in response to increased noise. Because native-data gradient SNR and oracle gains remain unobserved, these diagnostics support the qualitative mechanism and estimator response rather than establishing exact gain optimality.

\subsection{Forecasting Effects Across Regimes and Horizons}
\label{sec:accuracy}
\begin{figure*}[!h]
  \centering
  \includegraphics[width=380pt]{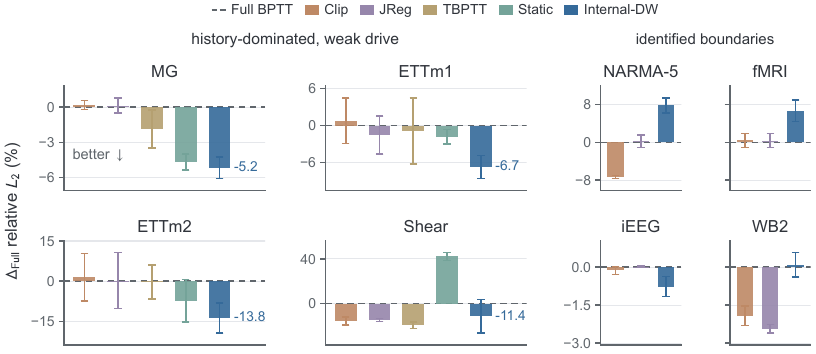}
  \setlength{\abovecaptionskip}{2pt}
\caption[Forecasting across predictive regimes.]{
  \textbf{Forecasting across predictive regimes.}
  Bars show the mean percentage change in relative $L_2$ from full BPTT,
  and error bars show the standard deviation across three matched seeds;
  negative values indicate improvement. TBPTT and Static are additionally
  evaluated on the four history-dominated, weak-drive testbeds. The underlying
  relative $L_2$ values are reported in
  Appendix~\ref{app:quantitative-table}.
}
  \vspace{-10pt}
  \label{fig:assigned-estimator-performance}
\end{figure*}

\paragraph{Forecasting across predictive regimes.}
Forecasting outcomes follow the applicability diagnostic in
Fig.~\ref{fig:drive-history-boundary-map}. Fig.~\ref{fig:assigned-estimator-performance} shows that Internal-DW
improves all four history-dominated, weak-drive testbeds over full BPTT,
reducing relative $L_2$ by $5.17\%$--$13.76\%$. It outperforms
validation-selected TBPTT on Mackey--Glass, ETTm1, and ETTm2, and the Static
control on ETTm1, ETTm2, and shear flow. DW outperforms Clip and JReg on MG, ETTm1, and ETTm2, while their performance is similar on shear flow. Thus neither truncation nor a single
global damping coefficient consistently reproduces the gains.

At the identified boundaries, Internal-DW provides no consistent benefit.
NARMA's near-closed gains coincide with $7.87\%$ worse forecasting,
consistent with the sampler assigning predictable drive-dependent variation
to noise. iEEG and WeatherBench-2 remain comparable to full BPTT, while
movie fMRI worsens. A controlled Mackey--Glass intervention strengthens this
interpretation: adding strong observed drive while retaining substantial
history value moves the system across the regime map and reverses
Internal-DW from a $5.17\%$ improvement to a $7.71\%$ degradation relative
to full BPTT (Appendix~\ref{app:driven-mg-boundary}). {Detailed runtime measurements are reported in Appendix}%
~\ref{app:computation-time}.
\vspace{-10pt}
\paragraph{Fitted optimal training-horizon range.}
Longer training horizons add both conditional-mean signal and innovation,
creating a marginal trade-off that reliability weighting can reshape but not
eliminate. We vary the training horizon $K$ while evaluating all checkpoints
at a fixed forecast horizon $H$ across MG, ETTm1, ETTm2, and Shear.
Fig.~\ref{fig:fixed-horizon-k-sweep} and the seed-specific curves in
Appendix~\ref{app:per-seed-curve} show that $K^\star_{\rm DW}\ge K^\star_{\rm Full}$ for every matched seed and in the seed-averaged fits, with a lower fitted minimum error for Internal-DW on all four datasets. This is consistent with the reliability trade-off in Theorem~\ref{thm:optimal-routing}: suppressing transported innovation while preserving predictable delayed credit can increase the marginal value of extending the rollout. Internal-DW therefore extends or preserves the useful training-horizon range rather than improving every horizon pointwise.
\begin{figure*}[h!]
  \centering
  \setlength{\abovecaptionskip}{0pt}

  \includegraphics[width=\textwidth]{
    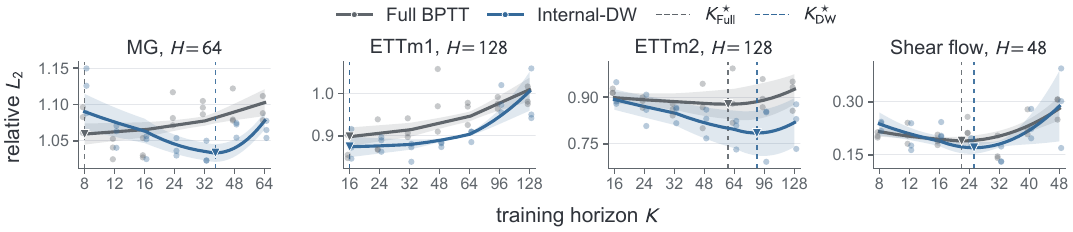
  }
  \vspace{-10pt}
\caption[Fitted optimal training horizons.]{
  \textbf{Fitted optimal training horizons.}
  Points show relative $L_2$ for checkpoints trained at different $K$ and
  evaluated at the fixed horizon $H$ shown in each panel; curves are
  upward-opening quadratic fits across matched seeds. Internal-DW shifts the
  fitted minimum toward larger $K$ in all four testbeds. The corresponding
  seed-specific fitted curves are shown in
  Appendix~\ref{app:per-seed-curve}.
}

  \label{fig:fixed-horizon-k-sweep}
  \vspace{-20pt}
\end{figure*}

\section{Discussion and Conclusion}
\label{sec:discussion}
\label{sec:conclusion}
\label{sec:limitations}
\label{sec:future-work}

Internal-DW minimizes a local quadratic route risk rather than the observed
rollout loss, so it does not guarantee a better training trajectory. Its gains
are only as accurate as the selected noise sampler, and solving each horizon
and layer independently ignores interactions after their contributions are
combined. These limitations suggest two extensions. Samplers that model
innovation conditional on history, observed drive, and their interaction---or
architectures that separate these pathways---could broaden the favorable
regime. Richer conditional samplers and justified pooling across related
channels, horizons, or layers could also reduce model mismatch and
finite-sample error.

Long-rollout BPTT can amplify distant gradient contributions without making
them reliable. Internal-DW addresses this mismatch by reliability-weighting
residual backward routes while preserving the forward rollout and all horizon
losses. The noise sampler provides the identifying structure needed to
estimate latent route reliability, allowing domain-appropriate samplers to be
plugged into a common routing framework. The known-SNR experiment connects
gradient amplification, route-risk reduction, and forecasting improvement,
while the application results show benefits primarily in history-dominated,
weak-drive regimes. Internal-DW can also extend or preserve the useful
training-horizon range, although it need not improve every $K$ pointwise and
can be neutral or harmful when usable history is weak or the sampler is
mismatched to dominant driven information. Thus Internal-DW is not a universal
stabilizer; it shows that long-horizon supervision can be retained without
trusting every backward contribution equally.

\subsection*{AI use statement}
We used generative AI tools to help refine hypotheses and
diagnostics, assist with figure preparation and improve the organization and wording of the manuscript.  The
authors checked the resulting derivations, references, figures, and claims against the cited sources and experiments.  The
authors take responsibility for the final content of the paper and all
AI-assisted artifacts.

\subsection*{Ethics statement}
This work performs secondary analysis of existing human neurophysiological
data and includes no new participant recruitment or data collection.  The
OpenNeuro iEEG release contains publicly shared de-identified recordings for
which participants or their legal guardians consented to research use and
public sharing, under the approval reported by the originating
study~\citep{keles2024multimodal}.  The HCP movie-fMRI data were collected with
informed consent and institutional approval~\citep{vanessen2013hcp}; we used
the open-access release under its data-use terms, did not use restricted
identifying variables, and made no attempt to re-identify participants.  The
experiments study forecasting methodology and are not intended for clinical
decision making or subject-level interpretation.

\subsection*{Reproducibility statement}
Section~\ref{sec:method} specifies the backward operator, calibration target,
and estimator assumptions.  The appendices provide derivations, noise-sampler
constructions, implementation details, dataset preprocessing and splits,
training hyperparameters, evaluation protocols, and seed-level results.
Checkpoints retain the route gains, estimated moments, probe counts, and
sampler configuration needed to audit calibration. 

\bibliography{iclr2027_conference}
\bibliographystyle{iclr2027_conference}

\appendix
\addtocontents{toc}{\protect\setcounter{tocdepth}{2}}
\clearpage
\begingroup
  \renewcommand{\contentsname}{Appendix Contents}
  \setcounter{tocdepth}{2}
  \tableofcontents
\endgroup
\clearpage

\section{Extended Related Work}
\label{app:related}

\paragraph{Sensitivity in autonomous and driven rollouts.}
Dynamical-systems analyses distinguish asymptotic instability from realized finite-time amplification, which can depend strongly on the starting state, forcing, observation error, and model error \citep{lorenz1996predictability,palmer2000predicting}. The same Jacobian products govern BPTT sensitivity, but the magnitude of a transported gradient also depends on the residual and its alignment with amplified directions~\citep{pascanu2013difficulty,mikhaeil2022difficulty}. Related work further shows that forecast errors can concentrate in unstable subspaces and that driven recurrent systems can exhibit generalized
synchronization with their inputs~\citep{trevisan2011aus,lu2018generalizedsync}. These results characterize the leverage of delayed information, not the reliability of the resulting gradient.

\paragraph{Credit assignment in reinforcement learning and differentiable
world models.}
Long-term credit assignment also arises in reinforcement learning, where much
of the literature asks how delayed returns should be attributed to earlier
states or actions. PIPPS uses total propagation to combine gradient estimators across
pathwise derivative depths, assigning greater weight to lower-variance
estimators~\citep{parmas2018pipps}. RUDDER redistributes delayed rewards through return
decomposition, while synthetic-return methods associate past states directly
with distant future rewards
~\citep{arjona2019rudder,raposo2021synthetic}.
Recent LLM-agent training methods similarly decompose long interaction
trajectories into training transitions or construct turn-level rewards
~\citep{luo2025agentlightning,tao2026trace}.
Differentiable world models provide a closer computational connection.
Stochastic value-gradient methods differentiate policy objectives through
learned dynamics, while Dreamer backpropagates value gradients through latent
imagined trajectories
~\citep{heess2015learning,hafner2020dream}.
Ma et al.\ further redesign the world-model architecture to provide more direct
long-range policy-gradient paths
~\citep{ma2024transformer}.
Collectively, these works alter how delayed returns are assigned or how policy
gradients are propagated through model rollouts. Internal-DW instead retains
the observed per-horizon targets, forward rollout, and all horizon losses, and
estimates the reliability of transported identity and nonlinear route VJPs.

\paragraph{Alternative recurrent-gradient estimators.}
TBPTT shortens the backward window, whereas ARTBP randomizes the cutoff and reweights surviving paths to obtain an unbiased estimator of full
BPTT at the cost of additional variance \citep{tallec2017unbiasing}. RTRL and its scalable variants propagate recurrent gradients online
\citep{williams1989learning,tallec2017unbiased,mujika2018approximating, benzing2019optimal}, while activation checkpointing trades recomputation for memory without changing the full BPTT gradient
\citep{griewank1992achieving,chen2016training,gruslys2016memory}. These methods change how a long-horizon gradient is computed or approximated; they do not estimate the reliability of individual route
contributions after transport.

\paragraph{Related methods that weight or modify backward signals.}
Kalman-inspired neural models learn filtering updates for latent-state estimation: Recurrent Kalman Networks maintain structured latent uncertainty, and KalmanNet learns a Kalman gain under partially known dynamics \citep{becker2019recurrent,revach2022kalmannet}.
Classical adaptive filtering also shows that process and observation covariances cannot generally be recovered from one total covariance without additional dynamical or rank assumptions \citep{mehra1970identification,belanger1974noise,
odelson2006autocovariance}. This identification principle motivates the explicit sampler contract
in Appendix~\ref{app:noise-nonidentifiability}. Synthetic gradients and related local methods instead learn or approximate module-level backward messages \citep{jaderberg2017decoupled,weber2019credit}, while Highway-BP uses residual shortcuts to approximate ordinary backpropagation more efficiently \citep{fagnou2025accelerated}.
Internal-DW retains the rollout VJPs and applies covariance-derived reliability weights to routes within each residual merge.

\section{Why Delayed Gradients Accumulate}
\label{app:theory-mech}

\label{sec:shared-gradient-accumulation}

The sensitivity in Equation~\ref{eq:horizon-gradient-sum} obeys the exact
forced recurrence
\begin{equation}
  D_{t,k+1}=J_{t+k}D_{t,k}+G_{t+k},
  \qquad D_{t,0}=0.
  \label{eq:forced-sensitivity-recurrence}
\end{equation}
Thus contraction of an old path does not remove the new parameter contribution
injected at the next rollout step.  To isolate the resulting long-horizon
regimes, consider a stationary local linearization and one parameter direction,
\begin{equation}
  z_{k+1}=Jz_k+\mu+\epsilon_k,
  \qquad
  \E[\epsilon_k]=0,\qquad
  \E[\epsilon_k\epsilon_j^\top]=\mathbf{1}\{k=j\}Q .
  \label{eq:stationary-forced-direction}
\end{equation}
Here $\mu$ is the coherent part of the local parameter injection and
$\epsilon_k$ is its temporally uncorrelated component.  Iteration gives
\begin{equation}
  \E[z_k]=\sum_{h=0}^{k-1}J^h\mu,
  \qquad
  \operatorname{Cov}(z_k)
  =\sum_{h=0}^{k-1}J^hQ(J^h)^\top .
  \label{eq:forced-direction-moments}
\end{equation}
The second expression is a finite controllability Gramian.  If
$\rho(J)<1$, both sums converge and
the root-mean-square (RMS) sensitivity approaches a
finite plateau.  A
unit-eigenvalue mode excited by $\mu$ accumulates in the mean, while any
semisimple unit-modulus mode excited by $Q$ accumulates in the variance.  If
$\rho(J)>1$ and an unstable mode is excited, the corresponding moment grows
exponentially.  A nontrivial Jordan block on the unit circle can produce still
faster polynomial growth.  These statements are matrix results; no scalar
reduction is required.

The scalar case makes the limits explicit.  For a nonnegative gain $a$, coherent injection
$\mu$, and innovation variance $\sigma^2$,
\begin{equation}
  \E[z_k^2]
  =\mu^2\!\left(\sum_{h=0}^{k-1}a^h\right)^2
   +\sigma^2\sum_{h=0}^{k-1}a^{2h},
  \quad
  \sum_{h=0}^{k-1}a^h=\frac{1-a^k}{1-a},
  \quad
  \sum_{h=0}^{k-1}a^{2h}=\frac{1-a^{2k}}{1-a^2},
  \label{eq:scalar-forced-rms-regimes}
\end{equation}
with the continuous limits $k$ at $a=1$.  Hence $|a|<1$ gives saturation,
$a=1$ gives linear accumulation in the coherent component and $\sqrt{k}$ RMS
growth in the uncorrelated component, and $|a|>1$ gives exponential growth.
Equal deviations around unity are also asymmetric at long horizons: for
$0<\delta<1$ and $k>1$,
\begin{equation}
  (1+\delta)^k-1
  >1-(1-\delta)^k,
  \label{eq:expansion-contraction-asymmetry}
\end{equation}
so a comparably expansive mode eventually dominates a comparably contractive
one.  This comparison does not assert that learning makes the two deviations
equally likely.

These regimes describe the shared-parameter sensitivity.  They carry
over to the gradient in Equation~\ref{eq:horizon-gradient-sum} when
the forecast residual excites the corresponding transported modes.
Time-varying Jacobians, correlated injections, and cancellation can change the
rate, so these are sufficient regimes rather than a claim that every nonlinear
autoregressive model produces large gradients.

\section{Internal-DW: Theory and Implementation}
\label{app:internal-dw-implementation}

This section first proves the optimality result underlying Internal-DW and
then gives the complete training and calibration procedure.
Algorithm~\ref{alg:dual-wiener} summarizes how ordinary training steps
alternate with fully open calibration probes that update the route moments
and gains. The remaining subsections specify the routing placement, probe
construction, constrained solve, and calibration cost.

\subsection{Proof of Theorem~\ref{thm:optimal-routing}}
\label{app:optimal-routing-proof}

\begin{proof}
Recall that $d=s+n$, where $s=\E[d\mid\mathcal C]$ and $n=d-s$.
Hence $\E[n\mid\mathcal C]=0$, and therefore
\[
\E[s^\top n]
=
\E\!\left[s^\top\E[n\mid\mathcal C]\right]
=
0.
\]
For any $w\in\mathbb R^2$,
\[
dw-s\mathbf 1
=
s(w-\mathbf 1)+nw.
\]
Using this orthogonality,
\begin{align*}
\mathcal R(w)
&=
\E\!\left[\left\|dw-s\mathbf 1\right\|_2^2\right] \\
&=
(w-\mathbf 1)^\top \E[s^\top s](w-\mathbf 1)
+
w^\top \E[n^\top n]w \\
&=
(w-\mathbf 1)^\top P(w-\mathbf 1)
+
w^\top Rw,
\end{align*}
which gives Equation~\ref{eq:dual-wiener-risk}.

Expanding and using $T=P+R$ gives
\[
\mathcal R(w)
=
w^\top T w
-
2w^\top P\mathbf 1
+
\mathbf 1^\top P\mathbf 1.
\]
Since $T\succ0$, $\mathcal R$ is strictly convex, and its unique
unconstrained minimizer satisfies
\[
\nabla_w\mathcal R(w)
=
2Tw-2P\mathbf 1
=
0.
\]
Thus
\[
w_{\rm unc}=T^{-1}P\mathbf 1.
\]

Using $Tw_{\rm unc}=P\mathbf 1$, completing the square yields
\[
\mathcal R(w)
=
\mathcal R(w_{\rm unc})
+
(w-w_{\rm unc})^\top T(w-w_{\rm unc}).
\]
Therefore
\[
w^\star
=
\arg\min_{w\in[0,1]^2}
(w-w_{\rm unc})^\top T(w-w_{\rm unc}),
\]
so $w^\star$ is the $T$-metric projection of $w_{\rm unc}$ onto
$[0,1]^2$.

Finally, $\mathbf 1\in[0,1]^2$, so optimality of $w^\star$ gives
\[
\mathcal R(w^\star)\le\mathcal R(\mathbf 1).
\]
From the risk decomposition,
\[
\mathcal R(\mathbf 1)=\mathbf 1^\top R\mathbf 1,
\]
and hence
\[
\mathcal R(w^\star)\le\mathbf 1^\top R\mathbf 1.
\]
Because $\mathcal R$ is strictly convex, the constrained minimizer is unique;
consequently, the inequality is strict whenever $\mathbf 1$ is not itself
optimal.
\end{proof}

For a single route, $T=P+R$ is scalar, so the unconstrained solution reduces
immediately to $w_{\rm unc}=P/(P+R)$.

\begin{algorithm}[h]
\caption{Internal-DW training and calibration}
\label{alg:dual-wiener}
\small
\begin{algorithmic}[1]
\REQUIRE Forecaster $f_\theta$, sampler $\mathcal Q_d$, and stored
$\{w_{k,\ell},\widehat T_{k,\ell},\widehat R_{k,\ell}\}$
\STATE Use $w_{k,\ell}=(1,1)^\top$ during the fully open warm-up
\FOR{training minibatch $b$}
    \STATE Roll out $K$ steps and form the task loss
    \IF{$b$ is a calibration minibatch after warm-up}
        \STATE Form $q_{\rm tot}$ and $q_{\rm noise}$ and run two fully open VJP probes
        \STATE Stage updates of $\widehat T$, $\widehat R$, and
        $\widehat P\leftarrow\Pi_{\rm PSD}(\widehat T-\widehat R)$
        \STATE Stage new gains by minimizing
        Equation~\ref{eq:dual-wiener-risk} over $w\in[0,1]^2$
    \ENDIF
    \STATE Backpropagate the task loss with the previously committed gains
    \STATE Update $\theta$ and commit any staged sampler statistics, moments, and gains
\ENDFOR
\end{algorithmic}
\end{algorithm}

\subsection{Where the Route Weights Are Applied}
\label{app:routing-placement}

At each forecast step and residual layer, two forward-identity autograd nodes
separately weight the identity and nonlinear backward messages, record their
unscaled VJPs, and leave the forward value unchanged. The recurrent-state path
uses the nonlinear coefficient. Routing is confined to internal residual
merges; the outer autoregressive update remains fully open because a single
gain there would combine paths represented separately by the two-route
estimator.
\subsection{How the Route Moments Are Calibrated}
\label{app:route-moment-calibration}

For calibration sample $i$ and forecast horizon $k$, the total-moment and
noise probes are
\begin{equation}
q_{\rm tot}
=
\frac{1}{2}\operatorname{mean}_{i,k}
\left\|\hat x_{i,k}-x^\star_{i,k}\right\|_2^2,
\qquad
q_{\rm noise}
=
\operatorname{mean}_{i,k}
\hat x_{i,k}^{\top}\widetilde\epsilon_{i,k},
\qquad
\widetilde\epsilon_{i,1:K}\sim\mathcal Q_d.
\label{eq:gaussian-noise-probe}
\end{equation}
{The quadratic $q_{\rm tot}$ need not equal the task loss used for
parameter updates. Both probes are differentiated with all route gains set
to one solely to estimate route moments and calibrate the Wiener gains;
their gradients are not accumulated into the task gradients used to update
model parameters.} Because route moments depend on the probe gradient, transfer across different probe and task losses is an estimator approximation rather than a theoretical guarantee. For the
residual block
\[
y_{k,\ell}
=
x_{k,\ell}+F_\ell(x_{k,\ell}),
\]
the recorded two-route message for a scalar probe $q$ is
\begin{equation}
d_{k,\ell}(q)
=
\left[
\underbrace{\nabla_{y_{k,\ell}}q}_{d^I_{k,\ell}(q)}
\quad
\underbrace{
J_{F_\ell}(x_{k,\ell})^\top\nabla_{y_{k,\ell}}q
}_{d^F_{k,\ell}(q)}
\right].
\label{eq:recorded-route-messages}
\end{equation}
Accordingly,
\[
d^{\rm tot}_{k,\ell}=d_{k,\ell}(q_{\rm tot}),
\qquad
d^{\rm noise}_{k,\ell}=d_{k,\ell}(q_{\rm noise}).
\]
These messages update the EMAs in
Equation~\ref{eq:plugin-route-moments}. If the block input has dimension
$D_\ell$, the implementation accumulates
\[
\frac{1}{D_\ell}
(d^{\rm tot}_{k,\ell})^\top d^{\rm tot}_{k,\ell}
\qquad\text{and}\qquad
\frac{1}{D_\ell}
(d^{\rm noise}_{k,\ell})^\top d^{\rm noise}_{k,\ell}.
\]
The common factor $1/D_\ell$ scales $\widehat T$, $\widehat R$, and
$\widehat P$ equally and therefore does not change the Wiener solution.

The probes estimate local route moments under the selected sampler; the
resulting Wiener objective is a local MMSE criterion, not a claim of global
optimality for the nonlinear training objective.

\subsection{How the Two Route Weights Are Solved}
\label{app:route-weight-solve}

Each symmetric $2\times2$ route matrix is stored as three scalars
$(II,IJ,JJ)$. Before solving, $\widehat T$ and $\widehat R$ are symmetrized
and projected onto the positive-semidefinite cone, after which
\[
  \widehat P
  =
  \Pi_{\rm PSD}(\widehat T-\widehat R).
\]
A numerical ridge of
$10^{-6}\operatorname{tr}(\widehat P+\widehat R)$ is added only to stabilize
the two-dimensional inverse.

The constrained solution is not obtained by clipping the unconstrained one,
because correlations between routes can make coordinatewise clipping
suboptimal. Instead, the implementation evaluates the feasible unconstrained
optimum, the optimum on each edge, and all four corners of $[0,1]^2$, and
selects the point with minimum quadratic risk.

\subsection{When Calibration Runs and What It Costs}
\label{app:calibration-cost}

Both calibration probes temporarily open every route and use
\texttt{autograd.grad} without accumulating parameter gradients. The task
backward uses the previously committed gains; newly estimated moments and
gains are committed only afterward. This prevents probe noise or an in-place
gain change from entering the current task update.

A calibration minibatch adds two VJP traversals. Soft gains do not prune the
computational graph, so the current implementation adds periodic computation
but does not reduce activation memory. Checkpoints store the gains, moments,
probe counts, and sampler configuration.

\section{How the Noise Samplers Are Constructed}
\label{app:noise-sampler-constructions}

Throughout this section, \(x_{i,t}\in\mathbb R^D\) denotes observed sequence
\(i\) at time \(t\), \(r\) is a rollout origin, and \(k\in\{1,\ldots,K\}\)
is a forecast horizon.  Sampler \(d\) supplies a prediction
\(\hat x^{(d)}_{i,r,k}\) of the common target \(x_{i,r+k}\), forms
\(e^{(d)}_{i,r,k}=x_{i,r+k}-\hat x^{(d)}_{i,r,k}\), and returns a
zero-mean trajectory
\(\widetilde\epsilon^{(d)}_{i,r,1:K}\sim\mathcal Q_d\).
Thus \(d\) only labels the sampler, while \(\mathcal Q_d\) is the distribution
of its final noise trajectories.
The same draw enters \(q_{\rm noise}\) in
\Eqref{eq:gaussian-noise-probe}; thus every construction below has the common
pipeline \(x\rightarrow\hat x^{(d)}\rightarrow
e^{(d)}\rightarrow\widetilde\epsilon^{(d)}\).

\begin{table}[H]
\centering
\caption{\textbf{Noise samplers used in the experiments.}  The structural
assumption changes only the noise sampler; route-moment estimation and the
Wiener solve are shared.}
\label{tab:identification-summary}
\scriptsize
\renewcommand{\arraystretch}{1.05}
\setlength{\tabcolsep}{3.2pt}
\begin{tabular}{
  >{\raggedright\arraybackslash}p{0.20\textwidth}
  >{\raggedright\arraybackslash}p{0.35\textwidth}
  >{\raggedright\arraybackslash}p{0.35\textwidth}}
\toprule
\textbf{Data} & \textbf{Assumed structure} &
\textbf{Sampler \(\mathcal Q_d\) and retained dependence} \\
\midrule
\multicolumn{3}{l}{\textit{DW-Generic}} \\
Mackey--Glass, NARMA, ETTm1, ETTm2 &
Centered residuals &
{Lagged residual template} (\Eqref{eq:generic-noise-sampler-appendix});
coordinates and horizons \\
\midrule
\multicolumn{3}{l}{\textit{DW-Prior}} \\
iEEG theta &
Long-range temporal dependence~\citep{linkenkaer2001long};
cross-electrode factor model (ours) &
Cross-fitted factor model; electrodes and horizons \\
\addlinespace[1pt]
HCP movie fMRI &
Shared stimulus response~\citep{chen2015reduced,nastase2019measuring};
AR residual model (ours) &
Subject cross-fit + AR model; regions and horizons \\
\addlinespace[1pt]
Shear flow, WeatherBench-2 &
Approximate translation stationarity, motivated by stationary-flow spectral
analysis~\citep{adrian1988stochastic,towne2018spectral} &
Lagged 2-D spectrum; spatial modes, horizons independent \\
\bottomrule
\end{tabular}
\end{table}

\begin{algorithm}[H]
\caption{Constructing a noise-sampler module}
\label{alg:prior-innovation}
\begin{algorithmic}[1]
\REQUIRE Training data, module type, and any predeclared domain structure

\IF{DW-Generic}
\STATE Center the last committed residual template using the committed EMA mean
\STATE Define $\mathcal Q_d$ using a trajectory-wide random sign
\STATE After each training minibatch's task backward, replace the residual
template and update the EMA mean

\ELSIF{template-based DW-Prior}
\STATE Split subjects or separated time blocks into folds
\FOR{fold $f$}
\STATE Fit the sampler predictor without fold $f$
\STATE Store
$e^{(d)}_{i,r,k}
=x_{i,r+k}-\hat x^{(d,-f)}_{i,r,k}$
for held-out origins
\ENDFOR
\STATE Define $\mathcal Q_d$ as a whole-trajectory held-out-residual bootstrap

\ELSIF{spatial DW-Prior}
\STATE From the previous minibatch, update per-horizon/channel residual
means, variances, and 2-D Fourier powers
\STATE Define $\mathcal Q_d$ as the corresponding random-phase Gaussian sampler
\ENDIF

\RETURN $\mathcal Q_d$, which generates
$\widetilde\epsilon^{(d)}_{i,r,1:K}$, its fitted state, and the declared
assumption
\end{algorithmic}
\end{algorithm}

\subsection{Generic Sampler: Resampling Forecast Residuals}
\label{app:generic-sampler}
\begingroup
The idea is to use recent forecast errors as a proxy for unpredictable
innovation. For each output coordinate and forecast horizon, we
center a stored residual by subtracting an exponential moving average of
past residuals across training minibatches. This retains deviations from
the estimated average error. We then either keep or reverse the sign of
the entire centered trajectory, with equal probability. The resulting
noise samples have zero mean, while their magnitudes and pairwise products
across coordinates and horizons remain unchanged. Using one sign for the
whole trajectory preserves these joint error patterns. Propagating the
samples backward lets us estimate how such output variation affects the
internal gradients.

The implementation stores one residual vector per
forecast horizon, forming a single trajectory template. For state dimension \(D\) and maximum
horizon \(H_{\max}\), this template has \(H_{\max}D\) entries.
During training minibatch \(n\), the first residual sample observed at each
horizon is detached and staged as its replacement; additional samples
contribute to the residual moments but do not replace that staged sample.
Shuffled training batches and randomly sampled rollout origins provide
variation in the selected trajectory across minibatches.

Here \(x_{i,t}\in\mathbb R^D\) denotes observed sequence \(i\) at time
\(t\), \(r\) is a rollout origin, and \(k\in\{1,\ldots,K\}\) is a
forecast horizon in a \(K\)-step training rollout. The index \(n\) denotes
the training minibatch, and \({\rm G}\) labels the generic sampler.
The detached prediction \(\hat x^{({\rm G},n)}_{i,r,k}\) targets
\(x_{i,r+k}\), giving the residual
\[
  e^{({\rm G},n)}_{i,r,k}
  =x_{i,r+k}-\hat x^{({\rm G},n)}_{i,r,k}.
\]

Let \(\mathcal B_n\) list the sequence--origin pairs observed in minibatch
\(n\), with repeated observations counted separately, and let
\((I_n,R_n)\) be its first pair. The stored template
\(\check e^{({\rm G})}_{n,k}\) takes this pair's residual, while the
running mean \(\mu^{({\rm G})}_{n,k}\) uses all residuals in the minibatch:
\[
  \begin{aligned}
  \check e^{({\rm G})}_{n,k}
    &=e^{({\rm G},n)}_{I_n,R_n,k},\\[3pt]
  \mu^{({\rm G})}_{n,k}
    &=\beta\mu^{({\rm G})}_{n-1,k}
      +\frac{1-\beta}{|\mathcal B_n|}
        \sum_{(j,s)\in\mathcal B_n}e^{({\rm G},n)}_{j,s,k}.
  \end{aligned}
\]
Here \(|\mathcal B_n|\) is the number of observed pairs, and
\(\beta=0.99\) is the EMA coefficient. Each horizon's mean is initialized
from its first observed minibatch mean. After each training minibatch's
task backward, the staged template overwrites the previous template and
the EMA moments are committed. This update also occurs on minibatches
without a calibration probe and when gradients are accumulated across
several minibatches before an optimizer step. One additional detached
template is held temporarily while the update is staged.

To generate noise during minibatch \(n\), the probe centers the previously
committed template using the previous EMA mean and multiplies it by a
random sign \(\xi_{i,r}\), drawn independently for each probe sample with
equal probability of \(+1\) and \(-1\):
\begin{equation}
  \widetilde\epsilon^{({\rm G})}_{i,r,k}
    =\xi_{i,r}\bigl(\check e^{({\rm G})}_{n-1,k}
                    -\mu^{({\rm G})}_{n-1,k}\bigr).
  \label{eq:generic-noise-sampler-appendix}
\end{equation}
The same sign is shared across output coordinates and horizons.
Collecting the samples over \(k=1,\ldots,K\) gives
the noise trajectory
\(\widetilde\epsilon^{({\rm G})}_{i,r,1:K}\sim\mathcal Q_{\rm G}\),
where \(\mathcal Q_{\rm G}\) denotes the generic sampler's noise distribution.
The committed template is shared by the probe samples; coordinates and
horizons are not resampled separately.
Noise probes at a horizon begin only after its template is initialized.
The identifying approximation remains that centered forecast residuals
represent unpredictable innovation: any predictable error remaining
in those residuals is counted as noise.
\endgroup

\subsection{Prior Samplers: Predicting Reproducible Structure First}
\label{app:crossfit-samplers}
The two neural samplers first predict structure that should repeat across
separated time blocks or subjects, then treat the held-out prediction residual
as noise.  Cross-fitting ensures that each residual comes from a predictor
that was not trained on that held-out block or subject.
For a sequence--origin pair \((i,r)\) in fold \(f\), sampler \(d\) predicts
the same targets using a model fit without that fold.  Let
\(\mathcal B_d\) be the resulting bank of cross-fitted residual trajectories.
Then
\begin{equation}
  \begin{aligned}
    e^{(d)}_{i,r,k}
      &=x_{i,r+k}-\hat x^{(d,-f)}_{i,r,k},
    &
    \bar e^{(d)}_{k}
      &=\frac{1}{|\mathcal B_d|}
        \sum_{(j,s)\in\mathcal B_d}e^{(d)}_{j,s,k},\\[2pt]
    \widetilde\epsilon^{(d)}_{i,r,1:K}
      &=\xi_{i,r}
        \bigl(e^{(d)}_{J,S,1:K}-\bar e^{(d)}_{1:K}\bigr)
  \end{aligned}
  \label{eq:prior-template-bootstrap}
\end{equation}
\((J,S)\) is sampled from \(\mathcal B_d\), and the independent sign
\(\xi_{i,r}\) is \(+1\) or \(-1\) with equal probability.  The same sampled
trajectory and sign are used at every horizon.  This preserves
the empirical coordinate and cross-horizon covariance while making
\(\widetilde\epsilon^{(d)}\) zero mean.  Each reported bank contains at most
1024 training-only cross-fitted trajectories sampled without replacement.

\subsubsection{iEEG: Predicting Shared Temporal Factors}
\label{app:ieeg-sampler}
The iEEG sampler assumes that reproducible activity is concentrated in a small
set of cross-electrode factors with predictable temporal evolution.  It
forecasts those factors from past activity and treats the remaining held-out
variation as noise.
Write each observed theta-envelope sequence as \(x_{i,t}\in\mathbb R^D\).
The training sequence--time pairs are divided into contiguous held-out blocks
\(\mathcal H_f\).  For fold \(f\), the fit set \(\mathcal I_f\) excludes
\(\mathcal H_f\) and a purge gap on both sides, so no lag used by the
predictor crosses the held-out block.
From \(\mathcal I_f\), define
\begin{equation}
  \mu_f=\frac1{|\mathcal I_f|}
  \sum_{(i,t)\in\mathcal I_f}x_{i,t},
  \qquad
  C_f=\frac1{|\mathcal I_f|-1}
  \sum_{(i,t)\in\mathcal I_f}
  (x_{i,t}-\mu_f)(x_{i,t}-\mu_f)^\top .
  \label{eq:ieeg-factor-covariance}
\end{equation}
Let \(L_{f,\rho}\in\mathbb R^{D\times \rho}\) contain the top \(\rho\)
orthonormal
eigenvectors of \(C_f\), and let
\(a_{i,t}^{(f)}=L_{f,\rho}^{\top}(x_{i,t}-\mu_f)\) be the retained factor
scores.  Each
factor is fit independently with a ridge-regularized AR(\(p\)) model:
\begin{equation}
  \widehat\phi_{f,j}^{(p)}
  =\arg\min_{\phi\in\mathbb R^p}
  \sum_{(i,t)\in\mathcal I_f^{(p)}}
  \left(a_{i,t,j}^{(f)}-\sum_{\ell=1}^{p}\phi_{\ell}
  a_{i,t-\ell,j}^{(f)}\right)^2
  +\lambda s_{f,j}\|\phi\|_2^2,
  \label{eq:ieeg-factor-ar}
\end{equation}
where \(\mathcal I_f^{(p)}\) contains only targets whose complete lag window is
inside one fit segment, \(s_{f,j}\) is the mean diagonal entry of that factor's
lag Gram matrix, and \(\lambda=10^{-3}\).  If the fitted scalar AR has root
radius above \(0.995\), its coefficient vector is multiplied by the largest
\(\kappa\in[0,1]\) whose companion matrix has spectral radius at most \(0.995\).

For a held-out rollout origin \((i,r)\in\mathcal H_f\), the observed factor
history initializes the recursion
\begin{equation}
  \widehat a_{i,r+k,j}^{(f)}
  =\sum_{\ell=1}^{p}\widehat\phi_{f,j,\ell}^{(p)}
  \widehat a_{i,r+k-\ell,j}^{(f)},
  \qquad k=1,\ldots,K,
  \label{eq:ieeg-factor-forecast}
\end{equation}
with \(\widehat a_{i,r+k,j}^{(f)}=a_{i,r+k,j}^{(f)}\) for \(k\le0\).
The component outside the retained factor subspace is not set to zero.  It is
held at its last observed value
\begin{equation}
  b_{i,r,f}^{\perp}
  =(I-L_{f,\rho}L_{f,\rho}^{\top})(x_{i,r}-\mu_f),
  \qquad
  \hat x^{({\rm iEEG})}_{i,r,k}
  =\mu_f+L_{f,\rho}\widehat a_{i,r+k}^{(f)}+b_{i,r,f}^{\perp}.
  \label{eq:ieeg-prior-predictor}
\end{equation}
The residual trajectory placed in the common bank is therefore
\(e^{({\rm iEEG})}_{i,r,k}
=x_{i,r+k}-\hat x^{({\rm iEEG})}_{i,r,k}\), which
\Eqref{eq:prior-template-bootstrap} converts into
\(\widetilde\epsilon^{({\rm iEEG})}_{i,r,1:K}\).
The pair \((\rho,p)\) minimizes the held-out MSE averaged equally over horizons
and coordinates, with valid origins pooled across the purged folds:
\begin{equation}
  (\rho^\star,p^\star)
  =\arg\min_{\rho\in\mathcal R,\,p\in\mathcal P}
  \frac{1}{N_{\rho,p}KD}
  \sum_f\sum_{(i,r)\in\mathcal H_f}\sum_{k=1}^{K}
  \|e_{i,r,k}^{({\rm iEEG};\rho,p)}\|_2^2 .
  \label{eq:ieeg-prior-selection}
\end{equation}
The reported artifact uses five contiguous folds, a 128-step purge gap,
\(\mathcal R=\{8,16,32,64,80\}\), and
\(\mathcal P=\{8,16,32,64\}\); training-only cross-fitting selects
\((\rho^\star,p^\star)=(80,64)\).  Long-range temporal dependence in neural
  oscillation envelopes motivates the memory component
  ~\citep{linkenkaer2001long}; the cross-electrode factor model is our identifying
  assumption.

\subsubsection{Movie fMRI: Predicting Shared and Subject-Specific Responses}
\label{app:fmri-sampler}
The movie-fMRI sampler uses different subjects as repeated responses to the
same stimulus.  It predicts the response shared across subjects together with
the temporally predictable part of each subject's deviation, and treats the
remaining held-out variation as noise.
Let \(x_{i,t}\in\mathbb R^D\) be the parcellated response of subject \(i\) at
time \(t\), aligned because all subjects view the same movie.  We first remove
each subject's static mean and use the centered \(x_{i,t}\) throughout the
remaining construction:
\begin{equation}
  x_{i,t}
  \leftarrow x_{i,t}-\frac1T\sum_{u=1}^{T}x_{i,u}.
  \label{eq:fmri-subject-centering}
\end{equation}
Subjects, rather than time points, are assigned to cross-fitting folds.  For a
held-out subject fold \(f\), the shared response and subject deviations fitted
from the remaining subjects \(\mathcal S_{-f}\) are
\begin{equation}
  M_t^{(-f)}=\frac1{|\mathcal S_{-f}|}
  \sum_{a\in\mathcal S_{-f}}x_{a,t},
  \qquad
  \Delta x_{a,t}^{(-f)}=x_{a,t}-M_t^{(-f)}.
  \label{eq:fmri-shared-deviation}
\end{equation}
For every parcel \(c\), a separate AR(\(p\)) model is fit jointly across the
fit subjects:
\begin{equation}
  \widehat\phi_{f,c}^{(p)}
  =\arg\min_{\phi\in\mathbb R^p}
  \sum_{a\in\mathcal S_{-f}}\sum_{t=p+1}^{T}
  \left(\Delta x_{a,t,c}^{(-f)}-
  \sum_{\ell=1}^{p}\phi_{\ell}\Delta x_{a,t-\ell,c}^{(-f)}\right)^2
  +\lambda s_{f,c}\|\phi\|_2^2,
  \label{eq:fmri-deviation-ar}
\end{equation}
with the same relative ridge \(\lambda=10^{-3}\), Gram-scale \(s_{f,c}\), and
root-radius cap \(0.995\) as above.  Order \(p=0\) sets the predicted deviation
to zero and therefore reduces to the shared response alone.  For held-out
subject \(i\in\mathcal S_f\) and rollout origin \(r\), the observed deviation
history initializes the recursive AR prediction
\(\widehat{\Delta x}_{i,r+k}^{(-f)}\), and
\begin{equation}
  \hat x^{({\rm fMRI})}_{i,r,k}
  =M_{r+k}^{(-f)}+\widehat{\Delta x}_{i,r+k}^{(-f)},
  \qquad
  e^{({\rm fMRI})}_{i,r,k}
  =x_{i,r+k}-\hat x^{({\rm fMRI})}_{i,r,k}.
  \label{eq:fmri-prior-predictor}
\end{equation}
We select
\begin{equation}
  p^\star=\arg\min_{p\in\{0,1,2,4,8\}}
  \frac{1}{N_pKD}
  \sum_f\sum_{i\in\mathcal S_f}\sum_r\sum_{k=1}^{K}
  \|e_{i,r,k}^{({\rm fMRI};p)}\|_2^2,
  \label{eq:fmri-prior-selection}
\end{equation}
using five subject folds and training data only; the reported artifact selects
\(p^\star=8\).  Equation~\ref{eq:prior-template-bootstrap} then returns
\(\widetilde\epsilon^{({\rm fMRI})}_{i,r,1:K}\).  Prior work supports the use of
cross-subject alignment to capture stimulus-locked responses
~\citep{chen2015reduced,nastase2019measuring}; the diagonal AR model for
  predictable subject deviation is our additional identifying assumption.

\subsection{Spatial Fields: Matching Residual Spatial Structure}
\label{app:spatial-sampler}
The spatial sampler treats a residual map as structured noise rather than as
independent pixels.  It estimates the residual scale and spatial frequency
pattern separately for each horizon and channel, then uses them to shape a
white Gaussian field.
The Shear and WeatherBench-2 modules construct a different sampler online.
Let \(\mathcal B_n\) contain the sequence--origin pairs in calibration
minibatch \(n\).  For \((i,r)\in\mathcal B_n\), horizon \(k\), channel \(c\),
and grid location \(u\in\Omega\), define the detached residual
\[
  e^{({\rm sp},n)}_{i,r,k,c}(u)
  =x_{i,r+k,c}(u)-\hat x^{({\rm sp},n)}_{i,r,k,c}(u),
  \qquad |\Omega|=HW .
\]
Its coordinatewise mean and second moment are updated after the minibatch by
\begin{align}
  \mu_{n,k,c}(u)
  &=\beta\mu_{n-1,k,c}(u)+(1-\beta)
    \frac1{|\mathcal B_n|}
    \sum_{(i,r)\in\mathcal B_n}e^{({\rm sp},n)}_{i,r,k,c}(u), \\
  M_{n,k,c}(u)
  &=\beta M_{n-1,k,c}(u)+(1-\beta)
    \frac1{|\mathcal B_n|}
    \sum_{(i,r)\in\mathcal B_n}e^{({\rm sp},n)}_{i,r,k,c}(u)^2,
  \qquad \beta=0.99,
  \label{eq:field-residual-moments}
\end{align}
with the first observed minibatch used for initialization.  Current residuals
are standardized using the \emph{previous} committed moments,
\begin{equation}
  z^{(n)}_{i,r,k,c}(u)=
  \frac{e^{({\rm sp},n)}_{i,r,k,c}(u)-\mu_{n-1,k,c}(u)}
  {\sqrt{v_{n-1,k,c}(u)+\delta_{n-1,k}}},
  \quad
  v_{n-1,k,c}(u)=[M_{n-1,k,c}(u)-\mu_{n-1,k,c}(u)^2]_+,
  \label{eq:field-standardized-residual}
\end{equation}
where \(\delta_{n-1,k}=10^{-8}\max\{10^{-20},
|\Omega|^{-1}C^{-1}\sum_{c,u}M_{n-1,k,c}(u)\}\) is only a numerical floor.
Let \(\mathcal F\) be the orthonormal two-dimensional real FFT and let
\(w(\omega)=1\) for the DC and, when present, Nyquist columns and
\(w(\omega)=2\) for the other stored horizontal frequencies.  The minibatch
power and its unit-spatial-variance normalization are
\begin{equation}
  \check S_{n,k,c}(\omega)
  =\frac1{|\mathcal B_n|}\sum_{(i,r)\in\mathcal B_n}
  |\mathcal F z^{(n)}_{i,r,k,c}(\omega)|^2,
  \qquad
  \mathcal N(S)(\omega)
  =\frac{S(\omega)}{|\Omega|^{-1}\sum_{\omega'}
  w(\omega')S(\omega')} .
  \label{eq:field-power-estimate}
\end{equation}
The stored color spectrum is then
\begin{equation}
  S_{n,k,c}
  =\beta S_{n-1,k,c}+(1-\beta)\mathcal N(\check S_{n,k,c}),
  \label{eq:field-spectrum-ema}
\end{equation}
again initialized by the first available normalized spectrum.  At the next
calibration probe, draw independent white real fields
\(W_{i,r,k,c}(u)\sim\mathcal N(0,1)\) and return
\begin{equation}
  \widetilde\epsilon^{({\rm sp})}_{i,r,k,c}
  =\sqrt{v_{n,k,c}+\delta_{n,k}}\;\odot\;
  \mathcal F^{-1}\!\left[
  \sqrt{S_{n,k,c}}\;\odot\;\mathcal F W_{i,r,k,c}
  \right].
  \label{eq:field-prior-sampler}
\end{equation}
Collecting these fields over \(k=1,\ldots,K\) gives the sampler output
\(\widetilde\epsilon^{({\rm sp})}_{i,r,1:K}\sim\mathcal Q_{\rm sp}\).
The draw is zero mean and preserves coordinatewise residual scale and
per-horizon, per-channel spatial power, but not cross-channel or cross-horizon
covariance. Sending the same field through both routes retains their induced
cross-route covariance. The identifying assumption is approximate translation
stationarity on the regular grid
~\citep{adrian1988stochastic,towne2018spectral}.
Using only previously committed moments isolates the probe graph and assumes
that parameters vary slowly relative to the EMA window.

\subsection{How Sampler Outputs Enter Internal-DW}
\label{app:sampler-interface}
All constructions above provide the same object: a \(K\)-step output-noise
trajectory.  Internal-DW passes that trajectory through the same fully open
backward graph, so only the assumed noise distribution changes across
samplers.
Equations~\ref{eq:generic-noise-sampler-appendix},
\ref{eq:prior-template-bootstrap}, and~\ref{eq:field-prior-sampler}
give the corresponding distribution \(\mathcal Q_d\).  For every module, a draw
\(\widetilde\epsilon^{(d)}_{i,r,1:K}\sim\mathcal Q_d\) is inserted into
\(q_{\rm noise}\) in \Eqref{eq:gaussian-noise-probe}; the resulting
identity/nonlinear VJPs determine \(\widehat R_{k,\ell}\) through
\Eqref{eq:plugin-route-moments}.  The route probe, PSD projection, and
constrained Wiener solve are shared.  Blocked out-of-fold prediction
can reject the declared sampler
assumption when its residual remains predictable, but it cannot prove that the
residual is physically irreducible noise.

\section{Why a Noise Sampler Is Needed and How Its Errors Affect the Gains}
\label{app:noise-sampler-error}

\subsection{Proof of Route-Law Non-Identifiability}
\label{app:noise-nonidentifiability}

Equation~\ref{eq:dual-wiener-risk} depends on the decomposition
$T=P+R$. Fully open route gradients identify the total moment $T$, but
not its signal and noise components separately.
Proposition~\ref{prop:innovation-nonidentifiable} shows that this ambiguity
is structural rather than a finite-sample defect of the estimator.

\begin{proposition}[The open route law does not identify the split]
\label{prop:innovation-nonidentifiable}
Let $T\succ0$ be the covariance of an observed two-route vector $z$. For any
two distinct matrices $R_1,R_2$ satisfying $0\preceq R_i\preceq T$, define
independent Gaussian pairs
\[
n_i\sim\mathcal N(0,R_i),\qquad
s_i\sim\mathcal N(0,T-R_i),\qquad
z_i=s_i+n_i.
\]
Then $z_1$ and $z_2$ have the same complete distribution
$\mathcal N(0,T)$, while their unconstrained Wiener gains
$w_i=T^{-1}(T-R_i)\mathbf 1$ differ whenever
$(R_1-R_2)\mathbf 1\ne0$. Hence neither the total covariance nor, in
general, the complete law of the open route vector identifies its signal and
innovation covariances without additional structure.
\end{proposition}

\begin{proof}[Proof of Proposition~\ref{prop:innovation-nonidentifiable}]
For each $i\in\{1,2\}$, independence gives
\[
  \operatorname{Cov}(z_i)
  =
  \operatorname{Cov}(s_i)+\operatorname{Cov}(n_i)
  =
  (T-R_i)+R_i
  =
  T.
\]
Because $s_i$ and $n_i$ are Gaussian, their sum therefore satisfies
$z_i\sim\mathcal N(0,T)$. Thus $z_1$ and $z_2$ have the same complete
distribution despite having different signal--noise decompositions.

For decomposition $i$, the signal moment is $P_i=T-R_i$, so the
unconstrained Wiener gain is
\[
  w_i
  =
  T^{-1}P_i\mathbf 1
  =
  T^{-1}(T-R_i)\mathbf 1.
\]
Consequently,
\[
  w_1-w_2
  =
  T^{-1}(R_2-R_1)\mathbf 1,
\]
which is nonzero whenever $(R_1-R_2)\mathbf 1\ne0$. Hence the same
open-route law can imply different Wiener gains.
\end{proof}

No amount of additional sampling from the open-route distribution resolves
this ambiguity, because the indistinguishable decompositions agree at the
population-law level. Classical innovation-covariance estimators resolve it
only after imposing state-space, observation, repeated-measurement, or
covariance-rank assumptions
\citep{mehra1970identification,belanger1974noise,
odelson2006autocovariance}.
Internal-DW makes the analogous domain assumption through $\mathcal Q_d$ and
transports sampled output innovations through the fully open VJPs to estimate
the required route-noise covariance.

\subsection{When a Fixed Sampler Recovers the Intended Noise}
\label{app:sampler-consistency}

For residual-based samplers, we define the population residual targeted by the
sampler and state conditions under which its empirical noise moment converges
to that target. For domain $d$, let $c_t^{(d)}$ contain the observations
available to the noise predictor, and let $\mathcal G_d$ be its predeclared
predictor class. The best predictor in this class and its residual are
\begin{equation}
 q_d^*\in\arg\min_{q\in\mathcal G_d}
 \E\!\left[\lVert X_{t+1}-q(c_t^{(d)})\rVert_2^2\right],
 \qquad
 \epsilon_{t+1}^{(d)}=X_{t+1}-q_d^*(c_t^{(d)}).
 \label{eq:domain-projection-innovation}
\end{equation}
If \(\mathcal G_d\) contains the conditional mean, this residual is the
conditional innovation.  Otherwise it also contains predictable structure
that the chosen predictor class cannot represent.  The guarantee below
concerns this explicitly defined residual; it does not claim that model error
is physical noise.

\begin{proposition}[Consistency under a fixed sampler model]
Fix a domain $d$ and basis $B_d$. Suppose that
(i) training units are stationary and ergodic, with cross-fitting
performed over independent units or asymptotically separated mixing blocks;
(ii) the foldwise predictor satisfies
$\widehat q_d^{(-f)}\to q_d^*$ in $L_2$;
(iii) $B_d$ is fixed and bounded;
and (iv) the empirical covariance is consistent in the retained
finite-dimensional basis.
Then the covariance of the out-of-fold residuals converges to
\[
\Sigma_d^*
=
\operatorname{Cov}\!\left(
B_d^\top\epsilon_{t+1}^{(d)}
\right).
\]
For a fixed bounded linear VJP map applied to the centered
retained-basis residuals, the empirical-average route-noise moment
$\widehat R_{k,\ell}$ also converges to its population counterpart.
\end{proposition}

\begin{proof}
Assumption~(ii) makes each cross-fitted residual converge in $L_2$
to $\epsilon^{(d)}$. Boundedness of $B_d$ and the ergodic law of
large numbers then give convergence of its first two moments;
assumption~(iv) gives the covariance limit.
For a fixed linear VJP map, each entry of the transported
second-moment matrix is a linear function of the residual covariance,
so convergence follows directly.
\end{proof}
The proposition separates finite-sample estimation error from sampler-model
mismatch. More data reduce the former but cannot repair omitted conditioning
variables or an inadequate predictor class, because these change the
population residual being estimated. For covariance-based samplers, analogous
consistency requires consistent estimation of the moments retained by the
declared covariance model; misspecification of that model remains a
population-level error.

\subsection{How Sampler Error Changes the Estimated Gains}
\label{app:gain-error}

\paragraph{Why drive-dominated routes can favor the open operator.}
For one routed direction, write the open message as $d=s+n$, where $s$ is its
conditional signal, $n$ is orthogonal noise,
$P=\E\|s\|_2^2$, and $R=\E\|n\|_2^2$. Applying a scalar gain $c$ gives
\begin{align}
  \mathcal R(c)
  &= \E\|c(s+n)-s\|_2^2 \\
  &= (1-c)^2P+c^2R, \\
  \mathcal R(c)-\mathcal R(1)
  &= (1-c)^2P-(1-c^2)R.
  \label{eq:drive-boundary-risk-derivation}
\end{align}
Differentiation gives $c^\star=P/(P+R)$. Hence adding useful drive-dependent
signal moves the correctly estimated gain toward the open route; degradation
requires the applied gain to remain too small.

\paragraph{How moment error perturbs the two-route gain.}
Let $\widehat T=T+\Delta_T$ and $\widehat R=R+\Delta_R$, and suppose the
unconstrained population solution
\[
w_{\rm unc}=T^{-1}(T-R)\mathbf 1
\]
lies in the interior of $[0,1]^2$, so that $w^\star=w_{\rm unc}$.
When the perturbations are small enough that the box constraint and PSD
projection remain inactive,
\begin{equation}
\widehat w-w_{\rm unc}
=
T^{-1}\!\left[
\Delta_T(\mathbf 1-w_{\rm unc})-\Delta_R\mathbf 1
\right]
+
O\!\left(\|\Delta_T\|^2+\|\Delta_R\|^2\right).
\label{eq:plugin-gain-perturbation}
\end{equation}
This expansion follows by substituting
\[
(T+\Delta_T)^{-1}
=
T^{-1}-T^{-1}\Delta_TT^{-1}
+
O\!\left(\|\Delta_T\|^2\right)
\]
into the unconstrained plug-in solve and retaining the first-order terms.

Sampler-model mismatch enters through systematic error in $\Delta_R$, while
finite task and noise probes contribute estimation error through both
$\Delta_T$ and $\Delta_R$. If
$b_w=\E[\widehat w]-w_{\rm unc}$ and
$V_w=\operatorname{Cov}(\widehat w)$, then
\begin{equation}
\E\!\left[
\mathcal R(\widehat w)-\mathcal R(w_{\rm unc})
\right]
=
b_w^\top T b_w+\operatorname{tr}(T V_w).
\label{eq:plugin-excess-risk}
\end{equation}
Thus moment error contributes to excess local route risk through both gain bias
and gain variance. This remains a local route-risk statement and does not imply
an end-to-end forecasting guarantee.

\section{Controlled Known-SNR Study}
\label{app:known-snr-oracle}

This appendix specifies the oracle used in
Fig.~\ref{fig:known-snr-failure}.  It lets us measure the conditional-mean gradient and innovation noise separately; ordinary data reveal only their
sum. {For panels (a)--(c), forecasting parameters remain fixed, and gradients and route moments are measured with all backward routes fully open. In panel (c), only the DW controller updates its gains.}

\subsection{Separating Conditional Gradient Signal and Noise}
\label{app:known-snr-signal-noise}
The eight-dimensional process is
\begin{equation}
  x_{t+1}=A_{\rm AR}x_t+\xi_{t+1},\qquad
  A_{\rm AR}=\operatorname{diag}(.995,.98,.95,.9,-.995,-.98,-.95,-.9),
  \label{eq:known-snr-process}
\end{equation}
with independent Gaussian innovations
\(\xi_t\sim\mathcal N(0,\operatorname{diag}(1-a_1^2,\ldots,1-a_8^2))\),
where \(a_j=(A_{\rm AR})_{jj}\).
The variance choice makes every stationary coordinate have unit marginal
variance while retaining distinct positive and sign-alternating memory scales.
For an observed state \(x_t\), the simulator gives the conditional future
exactly:
\begin{equation}
  \bar x_{t+k}:=\E[x_{t+k}\mid x_t]=A_{\rm AR}^k x_t.
  \label{eq:known-snr-conditional-mean}
\end{equation}
All future-noise realizations in one comparison begin at this same observed
state, so the model prediction and its tangent map are held fixed while only
the target innovation changes.  \Eqref{eq:known-snr-conditional-mean}
is written in raw simulator coordinates; the realized target and its
conditional mean are both passed through the checkpoint's same affine training
normalization before the loss is differentiated.

Let
\(\mathcal L_k^{(b)}=K^{-1}\|\hat x_{t+k}-x_{t+k}^{(b)}\|_2^2\) be the horizon loss
for future-noise draw \(b\), and let
\(g_k^{(b)}=\nabla_\theta\mathcal L_k^{(b)}\) be its complete full BPTT parameter
gradient, including every path from horizon \(k\) to the shared parameters.
Because the prediction and its derivative are fixed under the conditioning and
the loss is quadratic,
\begin{equation}
  s_k:=\E_b[g_k^{(b)}\mid x_t]
  =\nabla_\theta K^{-1}
    \|\hat x_{t+k}-\bar x_{t+k}\|_2^2,\qquad
  n_k^{(b)}:=g_k^{(b)}-s_k,\quad \E_b[n_k^{(b)}]=0.
  \label{eq:known-snr-gradient-decomposition}
\end{equation}
Thus the clean full-window diagnostic target is
\(\mu=\sum_{k=1}^{K}s_k\).

The per-horizon reliability quantities include the SNR in the left panel of
Fig.~\ref{fig:known-snr-failure} and the noise fraction reported in the main
text:
\begin{equation}
  \operatorname{SNR}_k=
  \frac{\|s_k\|_2^2}{\E_b\|n_k^{(b)}\|_2^2},
  \qquad
  \rho_k=
  \frac{\E_b\|n_k^{(b)}\|_2^2}
       {\E_b\|g_k^{(b)}\|_2^2}.
  \label{eq:known-snr-gradient-snr}
\end{equation}
The first panel plots total-gradient RMS on the same parameter
coordinates at every horizon, normalized by its value at \(k=1\)
, while Equation~\ref{eq:known-snr-gradient-snr}
reports the signal--innovation split directly.

\subsection{Risk of Including More Forecast Steps}
\label{app:known-snr-prefix-risk}
Define
\(S_k=\sum_{j=1}^{k}s_j\) and
\(N_k^{(b)}=\sum_{j=1}^{k}n_j^{(b)}\).  The prefix update is
\(S_k+N_k^{(b)}\), whereas the clean target remains the full-window signal
\(\mu=S_K\).  Consequently,
\begin{equation}
  \mathcal R_{\le k}
  =\E_b\|S_k+N_k^{(b)}-S_K\|_2^2
  =\underbrace{\|S_k-S_K\|_2^2}_{\text{missing-signal bias}}
   +\underbrace{\E_b\|N_k^{(b)}\|_2^2}_{\text{innovation risk}}.
  \label{eq:known-snr-prefix-decomposition}
\end{equation}
The cross term vanishes by conditional centering.  Full BPTT is the endpoint
\(k=K\): it has zero missing-signal bias but need not have minimum total risk.
This equation is the formal meaning of the middle panel; its finite Monte Carlo
closure error is at most \(0.082\) of the full BPTT risk in the four reported
splits.

\paragraph{Probe protocol.}
The first two panels use \(K=32\)
, 15 independent test trajectories with one rollout start per
trajectory, 64 innovation draws per conditional state, and four independent
replicates.  All curves within a replicate use the same realized gradients.

\subsection{Local Route Risk and Forecasting Performance}
\label{app:known-snr-gain-risk}

This subsection details the protocols for
Fig.~\ref{fig:known-snr-failure}(c)--(d): panel (c) measures local route risk with forecasting parameters fixed, while panel (d) compares forecasting performance after separate training.

For Fig.~\ref{fig:known-snr-failure}(c), three trained full BPTT checkpoints (seeds 0, 1, and 2; $K=32$) each traverse their corresponding training set twice with forecasting parameters fixed. Only the DW controller updates: each batch uses previously committed gains, and
its calibration updates take effect on subsequent batches.

All routing settings in panel (c) are evaluated using the local route-risk objective $\mathcal R(w)$ from Theorem~\ref{thm:optimal-routing} (Eq.~\ref{eq:dual-wiener-risk}), with route moments measured on the fully open graph. Full BPTT corresponds to $w=\mathbf1$. The misplaced setting circularly shifts the DW gain pairs by half the horizon range within each layer. The noise-informed oracle reference computes gains from the exact signal moment $P$ and an estimate of $R$ obtained from 64 draws of the true process. Risk is evaluated with 64 independent draws.

For each seed, risk is summed over all measured local routes and 22 post-warm-up batches, then divided by the corresponding full BPTT risk sum. We report the mean and sample standard deviation of the three seed-level ratios.

Fig.~\ref{fig:known-snr-failure}(d) evaluates separately trained
full BPTT, gradient clipping, Jacobian regularization, TBPTT, and
Internal-DW with three matched seeds (0, 1, and 2) and training horizon
$K=32$. Each run uses its minimum-validation-loss checkpoint.
Clip, JReg, and TBPTT hyperparameters are selected by validation loss
on seed 0 and then fixed across seeds: clipping thresholds
$\{0.1,0.3,1.0\}$ select $0.1$, JReg coefficients $\{0.01,0.1,1.0\}$
select $0.1$, and TBPTT lengths $\{4,8,16\}$ select $4$.
Internal-DW uses fixed default settings without additional tuning.
Test results are not used for either selection.

For each seed, each of 15 test trajectories contributes 64 deterministic, uniformly spaced valid forecast origins. From a 16-step observed history at each origin, the model rolls out autoregressively for 48 steps without
state resets or ground-truth refreshes. We compute relative $L_2$ error using Eq.~\ref{eq:dense-multistart-rel-l2}, averaging equally over forecast steps, origins, and trajectories. We report the mean and sample standard deviation of the three seed-level scores.

\section{Datasets and Training Details}
\label{app:impl-details}

This appendix records the datasets and configurations behind the selected-estimator
comparison in Section~\ref{sec:results}.

\subsection{Datasets}
\label{app:datasets}

Table~\ref{tab:dataset-sizes} summarizes the dataset sizes
after preprocessing. The counted units differ across datasets; their
construction and split protocols are described below.

\begin{table}[!htbp]
\centering
\small
\setlength{\tabcolsep}{4pt}
\caption[Dataset sizes after preprocessing.]{{\textbf{Dataset sizes after preprocessing.}
Counts refer to trajectories, subjects, or sequence segments, as indicated.
State shape excludes the time axis and external-drive features; length is
measured in model steps.}}
\label{tab:dataset-sizes}
\begin{tabular}{@{}llccc@{}}
\toprule
Dataset & Counted unit & Train / Val. / Test & State shape & Length \\
\midrule
MG & Trajectory & 28 / 6 / 6 & $8$ & 2048 \\
NARMA-5 & Trajectory & 28 / 6 / 6 & $8$ & 2048 \\
iEEG & Chunk & {464 / 87 / 87} & {40 or 80} & {1024} \\
Movie fMRI & Subject & 111 / 24 / 24 & $400$ & 284 \\
ETTm1 / ETTm2 & Segment & 135 / 45 / 45 each & $7$ & 256 \\
Shear flow & Clip & 288 / 36 / 36 & $4\times64\times128$ & 64 \\
WeatherBench-2 & Segment & 205 / 22 / 22 & $29\times121\times240$ & 64 \\
\bottomrule
\end{tabular}
\par\smallskip
\begin{minipage}{\linewidth}
\footnotesize
iEEG comprises 29 recordings from 16 participants; counts are pooled across
participants, with models trained separately for each participant.
State dimensions exclude the 512-dimensional CLIP visual inputs.
Shear flow uses 32/4/4 source trajectories;
clips overlap within each trajectory but never cross trajectories or splits.
\end{minipage}
\end{table}

\paragraph{Mackey--Glass.}
For each coordinate, Mackey--Glass~\citep{mackey1977oscillation} follows
\[
\dot{x}(t)
=
\frac{\beta x(t-\tau)}
     {1+x(t-\tau)^{10}}
-\gamma x(t),
\]
with $\tau=30$, $\beta=.2$, and $\gamma=.1$.
We generate eight independent coordinates with identical dynamics and
independent initial histories. The equation is integrated by RK4 (fourth-order Runge–Kutta method) with an
internal step of $.1$, sampled every unit time, and run for a 1000-step
burn-in before retaining each trajectory.
We generate and cache 40 trajectories of length 2048 using a fixed
data-generation seed. For each matched experimental seed, trajectories are
randomly divided into 28/6/6 training/validation/test trajectories.
Output states are standardized using the mean and standard deviation of
the corresponding training trajectories only.

\paragraph{NARMA-5.}
We use a bounded variant of the NARMA benchmark~\citep{atiya2000new}.
For each coordinate, bounded NARMA-5 follows
\[
y_{t+1}
=
\tanh\!\left(
.3y_t
+.05y_t\sum_{j=0}^{4}y_{t-j}
+1.5u_{t-4}u_t
+.1
\right),
\qquad
u_t\sim\mathcal U[0,.5].
\]
The eight coordinates use independently sampled observed input channels, and
the first 200 steps are discarded as burn-in. For both datasets, we generate
and cache 40 trajectories of length 2048 using a fixed data-generation seed.
For each matched experimental seed, trajectories are randomly divided into
28/6/6 training/validation/test trajectories. Output states are standardized
using the mean and standard deviation of the corresponding training
trajectories only; the NARMA inputs retain their original scale.

\paragraph{iEEG.}
{
The iEEG experiment uses 29 movie-encoding recordings from 16 participants
in the dataset of \citet{keles2024multimodal}. The neural state consists of
macro-contact theta-power signals: 15 participants have 80 channels, and
CS62 has 40. Each participant also has aligned 512-dimensional CLIP visual
features, supplied as stimulus inputs.

{The electrophysiology data are distributed in NWB format through
DANDI (Dandiset 000623). Participants watched an approximately eight-minute
excerpt from Hitchcock's \emph{Bang! You're Dead}; the released macroelectrode
signals are sampled at 1000~Hz and bandpass filtered at 0.1--500~Hz
\citep{keles2024multimodal}. Our upstream preprocessing retains macroelectrode channels covering the amygdala, hippocampus, anterior cingulate cortex, pre-SMA, and vmPFC. These signals undergo notch filtering at 60~Hz and its
harmonics, 0.1-Hz high-pass filtering, and common-average referencing.
Theta power is computed using five-cycle Morlet wavelets at 4, 5, 6, 7,
and 8~Hz. Power at each frequency is z-scored over time before averaging
across frequencies. The resulting time courses are resampled to 500~Hz
and cropped to the movie interval before the subsequent 50-Hz model-input
preparation described below.}

For each participant, we split the movie timeline into 70/15/15
training/validation/test intervals. Repeated viewings use the same movie-time
boundaries, preventing the same movie segment from appearing in different
splits through different recordings. Repeated recordings are first truncated
to their common length. We exclude 256 steps (5.12~s) immediately before
each split boundary, not on both sides; no additional 2-s edge trimming is
applied. The existing FIF signals preserve recording-wide Morlet z-scoring
and noncausal filtering. Zero-phase FIR anti-alias downsampling also precedes
the split. Therefore, preprocessing is not entirely split-local or strictly
causal. At the final sampling rate of 50~Hz, each model step represents 20~ms.
The existing 512-dimensional visual features are sampled at 25 frames/s
and aligned by holding the most recent movie frame; audio is not included.

Each processed recording segment is divided into non-overlapping
1024-step chunks (20.48~s each), discarding incomplete trailing chunks.
Chunks never cross recording or split boundaries. Multiple recordings
from one participant are pooled only within their corresponding splits.
Across participants, this yields 464/87/87 training/validation/test chunks,
or 475,136/89,088/89,088 retained model steps. These are post-processing
counts, not raw sample counts or independent participant counts.

Neural channels and CLIP feature dimensions are standardized separately
for each participant using the mean and standard deviation computed
from that participant's training chunks only. The same statistics are
applied unchanged to validation and test chunks; no statistics or training
data are pooled across participants. Constant dimensions use a unit scale.
This training-only claim applies to the second-stage normalization, not to
the upstream processing described above. Pooled counts are given
in Table~\ref{tab:dataset-sizes}. Forecasting uses $K=64$ and test horizons
1--96 (up to 1.92~s); reliability diagnostics use horizons 1--64.
}

\paragraph{Movie fMRI.}
Movie fMRI uses the HCP 7T MOVIE1 AP acquisition~\citep{vanessen2013hcp}.
{Data were acquired using gradient-echo EPI with TR~=~1~s,
TE~=~22.2~ms, a flip angle of $45^{\circ}$, 1.6-mm isotropic voxels,
85 slices, and a multiband factor of 5, following the HCP 7T acquisition
protocol.\footnote{\url{https://www.humanconnectome.org/hcp-protocols-ya-7t-imaging}}}
We start from the
HCP-provided {minimally preprocessed and ICA-FIX-denoised}
\texttt{tfMRI\_MOVIE1\_7T\_AP\_hp2000\_clean.nii.gz}
volumes in MNINonLinear space, which have 1.6-mm isotropic resolution and
a 1-s TR. The 2-mm Schaefer2018 400-parcel, seven-network atlas~\citep{schaefer2018local} is resampled
to the HCP volume grid using nearest-neighbor interpolation. BOLD values are
then averaged within each parcel, and each parcel time series is z-scored
over time within the run.

Each subject sequence contains 20 pre-movie TRs, 244 MOVIE1 TRs, and
20 post-movie TRs. The accompanying movie-derived CLIP representation
contains 256 1664-dimensional features per TR; these are averaged over
the 256 tokens before entering the forecasting model. The packaged data
record a six-TR hemodynamic alignment lag, and no additional temporal shift
is applied by the dataset loader. Of the 159 subjects with complete inputs,
a seed-specific shuffle assigns 111/24/24 subjects to the
training/validation/test splits, so no subject appears in more than one
split.

\paragraph{ETTm1 and ETTm2.}
ETTm1 and ETTm2 contain 15-minute measurements from two electricity-transformer
records~\citep{zhou2021informer}. Each state contains oil temperature and six
load measurements, and all seven coordinates are forecast. The observed
future drive consists of six deterministic calendar coordinates: sine and
cosine encodings of hour of day, day of week, and day of year.

We use the standard chronological partition consisting of the first 12 months
for training, the next four months for validation, and the following four
months for testing. Each split is divided into non-overlapping 256-step units,
producing 135/45/45 training/validation/test units for each dataset. State and calendar
coordinates are standardized separately using per-coordinate statistics
computed from the training units only.

\paragraph{Shear flow.}
Shear flow uses the supplied train/validation/test partitions of
\texttt{shear\_flow} from The Well~\citep{ohana2024well} without reshuffling trajectories. Each
state contains the scalar tracer and pressure fields and the two components
of the velocity field. Scalar and vector fields are aligned and flattened
into four channels before deterministic spatial subsampling by a factor of
four, giving $64\times128$ spatial inputs in our experiments. Within each
upstream trajectory, we form 64-frame sequences with stride 16; sequences
never cross trajectories or dataset splits. We apply no temporal subsampling,
supply no external drive, and apply no additional dataset-level
standardization.
{We use the $Re=10^5$, $Sc=2$ configuration, containing
32/4/4 training/validation/test trajectories of 200 frames each.
With clip length 64 and stride 16, each trajectory yields nine clips,
giving 288/36/36 clips in total. Clips overlap within a trajectory,
but no trajectory is shared across splits.}

\paragraph{WeatherBench-2.}
WeatherBench-2~\citep{rasp2024weatherbench2} is extracted from the official $1.5^\circ$ ERA5 archive at
six-hour resolution on its $121\times240$ grid. The state contains
geopotential, temperature, zonal wind, meridional wind, and specific humidity
at 1000, 850, 700, 500, and 250 hPa, together with surface pressure,
2-m temperature, and 10-m zonal and meridional winds, for 29 channels in
total. The observed future drive consists of sine and cosine encodings of
hour of day and day of year.

We use 2007--2015 for training, 2016 for validation, and 2017 for testing,
corresponding to 13,148/1,464/1,460 time steps. Per-channel means and standard
deviations are computed over all training times and spatial locations and
then applied unchanged to every split. No additional spatial subsampling is
performed. {Because training requires holding rollout activations in memory for backpropagation, the training and validation splits are divided into
non-overlapping 64-frame segments (205/22 samples), bounding memory use per
batch as for the other testbeds. Within each training segment,
the burn-in window and $K$-step rollout target are drawn from a randomized
start position each training step (Appendix~\ref{app:training-details}), so successive epochs see varied alignments rather than a single fixed set of 205 examples. Test-time evaluation involves no backward pass and is not subject to this constraint: the reported test relative $L_2$ is computed with a continuous-timeline evaluator that draws dense rollout origins directly from the full test-year array and rolls each out autoregressively through $H_{\text{eval}} = \lceil 1.5K \rceil$ steps before averaging across origins (Appendix~\ref{app:evaluation-protocol}).}

\subsection{Implementation Details for Training}
\label{app:training-details}

Within each dataset, full BPTT and
Internal-DW use the same split, initialization seed, forward model, task loss,
optimizer update, validation rule, and evaluation code. Full BPTT keeps every
route open; activation checkpointing is used where needed to reproduce the
same gradient with lower memory. Internal-DW changes only the backward
operator and adds the calibration probes described in
Appendix~\ref{app:internal-dw-implementation}.

\paragraph{Settings shared across methods.}
{For iEEG, full BPTT and Internal-DW are trained separately
for each participant using three matched seeds (0, 1, and 2).
Model parameters and training data are not shared across participants;
this evaluates within-participant forecasting, not generalization to
unseen participants.} {For all sequence and field datasets, the rollout start within each loaded
segment is drawn uniformly at random at each training step (and held fixed
across ranks within a distributed step), rather than fixed at a single
alignment. For WeatherBench-2 this increases the effective diversity of
burn-in/target alignments seen from the 205 non-overlapping training
segments.}
All matched arms use Adam, train for at most 100
epochs with validation patience 20, and report the minimum-validation-loss
checkpoint.  The primary arms use global gradient-norm cap 1; stronger clipping
is stated with the corresponding control below.  Sequence models use
four residual Mamba blocks~\citep{gu2023mamba} with state dimension 16, convolution width 4,
expansion factor 2, a linear read-in, and a two-layer read-out, and predict an
increment added to the current state.  They use a 16-step input window, a
32-step burn-in, and 16 rollout starts.  Field models use a depth-4 residual
U-Net with channel multiplier 2 and a two-frame input.

\begin{table}[h]
\centering
\small
\setlength{\tabcolsep}{3.5pt}
\caption{\textbf{Dataset-specific training and calibration settings.}
Model numbers give the Mamba hidden width or U-Net base width. ``W/C'' gives
the number of fully open warm-up minibatches and the calibration period.}
\label{tab:dataset-training-settings}
\begin{tabular}{@{}llllll@{}}
\toprule
Data & Model & \(K\)/loss & LR/batch & Noise sampler & W/C \\
\midrule
MG, NARMA & Mamba-128 & 32/Rel-\(L_2\) & \(10^{-4}\)/32 & Generic & 8/4 \\
ETTm1/2 & Mamba-128 & 64/Rel-\(L_2\) & \(10^{-4}\)/32 & Generic & 8/4 \\
iEEG & Mamba-256 & 64/Rel-\(L_2\) & \(10^{-4}\)/32 & {Factor prior} & 8/4 \\
fMRI & Mamba-4096 & 64/Rel-\(L_2\) & \(10^{-4}\)/32 & {Subject-AR prior} & 8/4 \\
Shear & U-Net-32 & 32/Rel-\(L_2\) & \(3\!\times\!10^{-4}\)/4 & {Spectral prior} & 8/4 \\
WB2 & U-Net-32 & 48/Rel-\(L_2\) & \(10^{-4}\)/8 & {Spectral prior} & 2/1 \\
\bottomrule
\end{tabular}
\end{table}

The sequence and WeatherBench-2 runs use weight decay \(10^{-4}\) and the
learning rates in Table~\ref{tab:dataset-training-settings} without decay over the
100-epoch training window; shear uses zero weight decay and cosine decay to
\(10^{-6}\). The gains start fully open, with route- and residual-moment EMA
coefficients .95 and .99.

\paragraph{Selection of Clip, JReg, TBPTT, Static, and DW variants.} 
For each dataset, hyperparameters and variants are selected using seed-0
validation loss and then fixed across three matched seeds. Full BPTT, Internal-DW, JReg, and TBPTT use global gradient-norm clipping
with a threshold of $1.0$. The Clip baseline instead selects its
threshold from $\{0.1, 0.3, 1.0\}$ using seed-0 validation loss;
the selected threshold is then fixed across all three seeds.
JReg coefficients are selected from $\{0.01,0.1,1.0\}$. For iEEG, Clip and JReg selection uses the
equal-weight mean validation loss across participants. JReg uses target
norm $1$ and finite-difference scale $0.001$.
TBPTT candidate segment lengths are $\{8,16\}$ for
Mackey--Glass and shear flow and $\{16,32\}$ for ETTm1/2. The Static control ties every routed merge to $\alpha=m=c$
and selects $c\in\{0.3,0.6,0.9\}$. DW selection chooses between DW-Generic and DW-Prior. Table~\ref{tab:selected-control-hyperparameters} reports all selected values and variants. Calibration cost and update order are given in
Appendix~\ref{app:calibration-cost}.

\begin{table}[htbp]
\centering\footnotesize
\setlength{\tabcolsep}{2pt}
\caption{\textbf{Validation-selected hyperparameters and DW variants.}
Generic and Prior denote DW-Generic and DW-Prior, respectively.
A dash indicates that the control was not evaluated on that dataset.}
\label{tab:selected-control-hyperparameters}
\begin{tabular*}{\textwidth}{@{\extracolsep{\fill}}lcccccccc@{}}
\toprule
Parameter / variant & MG & ETTm1 & ETTm2 & Shear & NARMA-5 & iEEG & fMRI & WB2 \\
\midrule
Clip threshold & $1.0$ & $0.3$ & $0.3$ & $0.1$ & $0.1$ & $0.3$ & $1.0$ & $1.0$ \\
JReg coefficient & $1.0$ & $1.0$ & $1.0$ & $0.1$ & $1.0$ & $0.01$ & $1.0$ & $0.01$ \\
TBPTT segment length & $8$ & $32$ & $32$ & $8$ & -- & -- & -- & -- \\
Static gain $c$ & $0.6$ & $0.3$ & $0.6$ & $0.3$ & -- & -- & -- & -- \\
DW variant & Generic & Generic & Generic & Prior & Generic & Prior & Prior & Prior \\
\bottomrule
\end{tabular*}
\end{table}

\paragraph{Movie-fMRI sampler fitting.}
The fMRI noise sampler is fitted separately for each seed using only its
training subjects. Five-fold subject cross-fitting estimates a time-aligned
shared response from the fitting subjects and a coordinatewise
ridge-regularized AR model for subject-specific deviations. The AR order is
selected from $\{0,1,2,4,8\}$ by cross-fitted multi-horizon MSE and is 8 for
all three seeds. Length-64 residual trajectories from the held-out training
folds are retained intact as noise templates, preserving cross-parcel and
cross-horizon dependence.

\section{How History and Observed Drive Are Measured}
\label{app:history-drive-measurement}

\subsection[Probe Construction and Regime Scores]{{Probe Construction and Regime Scores}}
\label{app:regime-scores}

\paragraph{{Probe construction.}}
These probes are diagnostic tools and are not guaranteed to be optimal or
well suited to every dataset. We use their scores to obtain a coarse,
predictor-dependent indication of how readily observed drive and state
history can be used for prediction, rather than a precise measurement of
a dataset's intrinsic properties.

We use linear ridge regression and a nonlinear nearest-neighbor
(local-analog) regressor. The latter predicts by finding fitting examples
with inputs similar to the current input and averaging their target values.
Each feature is standardized using the fitting examples' mean and standard
deviation, with a standard-deviation floor of \(10^{-5}\). We select the
\(q\) nearest examples by Euclidean distance and average their targets
with equal weights, using all available examples if fewer than \(q\)
are available. Although this average is simple, the selected neighbors
change with the input, so the overall prediction is a nonlinear function
of the input. The nearest-neighbor predictor has no ridge penalty;
\(q\) controls the amount of local averaging.

The drive predictor uses observed future drive for NARMA and WeatherBench-2,
calendar encodings for ETTm1 and ETTm2, and a subject-disjoint estimate of
the movie-locked shared response for fMRI.
For the common probe, drive-only features concatenate the available future
drive up to the prediction horizon. The history-augmented probe also
concatenates \(W\) past states after PCA projection. For movie fMRI,
the history features instead contain PCA-projected deviations from the
shared movie response, and the predicted targets are future deviations
from that response. Table~\ref{tab:regime-probe-settings} gives the
history PCA dimensions, candidate widths, and neighbor counts.

\begin{table}[h]
\centering
\small
\caption{\textbf{Predictive-regime probe settings. }History width \(W\) is measured
in model time steps; \(q\) is the number of nearest neighbors.}
\label{tab:regime-probe-settings}
\begin{tabular}{lrrl}
\toprule
Data & History PCA dimension & \(q\) & Candidate \(W\) \\
\midrule
Mackey--Glass & 8 & 32 & \(\{1,2,4,8,16\}\) \\
NARMA-5 & 8 & 32 & \(\{1,2,4,8\}\) \\
iEEG & 32 & 32 & \(\{1,2,4,8,16\}\) \\
ETTm1 / ETTm2 & 7 & 32 & \(\{1,2,4,8,16,32\}\) \\
Shear flow & 32 & 32 & \(\{1,2,4,8\}\) \\
WeatherBench-2 & 32 & 24 & \(\{1,2,4,8\}\) \\
Movie fMRI & 16 & 32 & \(\{1,2,4,8,16\}\) \\
\bottomrule
\end{tabular}
\end{table}

The linear probe includes an intercept and uses a relative ridge coefficient
of \(0.01\): the diagonal regularization added to the normalized feature
Gram matrix is \(0.01\) times its mean diagonal entry. History width is
selected by minimizing the linear probe's validation MSE relative to the
drive-only MSE, averaged across the evaluated horizons; the nonlinear probe
uses the same selected width. Neighbor counts and the ridge coefficient
are fixed. After selecting the width, both predictors are fitted using
training plus validation examples and evaluated once on the test split;
the reported long-horizon scores average over \(k\ge8\). Feature-scaling
statistics are estimated from fitting examples and applied unchanged to
query examples. The iEEG procedure is performed separately for each
participant, and fMRI fitting and evaluation use disjoint subjects.

\paragraph{{Regime scores.}}
Let \(R_{\rm mean}\) be the held-out risk of a training-mean predictor,
\(R_D\) that of a predictor using the observed future drive, and \(R_{D+H}\)
that of the same predictor augmented with history.  Fig.~\ref{fig:drive-history-boundary-map}
uses
\begin{equation}
  \Psi_D=1-\frac{R_D}{R_{\rm mean}},
  \qquad
  \Psi_H=1-\frac{R_{D+H}}{R_D},
  \label{eq:drive-history-values}
\end{equation}
and reports \(\Delta\Psi_H=\Psi_H-\Psi_H^{\rm null}\), where the null uses
circularly shifted history.

Fig.~\ref{fig:drive-history-boundary-map} first averages the values across
linear and nonlinear probes and then clips negative values to zero because the
map uses only detected predictive value. Table~\ref{tab:predictive-regime-values} reports the corresponding unclipped scores.

\begin{table}[h]
\centering
\small
\setlength{\tabcolsep}{5pt}
\caption{\textbf{Unclipped predictive-regime probe values.}}
\label{tab:predictive-regime-values}
\begin{tabular}{lrrrr}
\toprule
& \multicolumn{2}{c}{Linear probe} & \multicolumn{2}{c}{Nonlinear probe} \\
\cmidrule(lr){2-3}\cmidrule(lr){4-5}
Data & \(\Psi_D\) & \(\Delta\Psi_H\) & \(\Psi_D\) & \(\Delta\Psi_H\) \\
\midrule
Mackey--Glass  & .000 &  .442 & .000 &  .370 \\
NARMA-5        & .857 & \(-.001\) & .283 & \(-.001\) \\
iEEG  & \(-3.2823\) & {.0770} &{\(-.1265\)} & {.0032} \\
Shear flow     & .000 &  .859 & .000 &  .544 \\
WeatherBench-2 & .219 & \(-.179\) & .232 &  .004 \\
ETTm1          & \(-.341\) & .697 & \(-.355\) & .110 \\
ETTm2          & \(-.144\) & .963 & \(-.019\) & .272 \\
Movie fMRI     & .230 &  .039 & .230 &  .016 \\
\bottomrule
\end{tabular}
\end{table}

\subsection{Interpreting the Regime Scores}
\label{app:regime-score-interpretation}
The raw scores measure held-out risk reduction and are not clipped at zero, so
they can be negative. \(\Psi_D<0\) means that the drive-based probe performs worse than the training-mean predictor, while \(\Delta\Psi_H<0\) means that aligned history performs worse than the shifted-history null.  We interpret either outcome as no detected predictive value under the probe, not as negative information.

For autonomous testbeds, the drive-only predictor
is the training mean, so \(\Psi_D=0\) by construction.

\subsection{Validation of the Predictive-Regime Probes}
\label{app:probe-validation}

We perform a controlled Mackey--Glass sweep to verify that the
predictive-regime probes respond to increasing observed drive.

For each of the eight independent coordinates, the driven dynamics are
\begin{equation}
\dot{x}(t)
=
\frac{\beta x(t-\tau)}{1+x(t-\tau)^{10}}
-\gamma x(t)+\lambda_D u(t),
\label{eq:probe-driven-mg}
\end{equation}
where $\tau=30$, $\beta=.2$, and $\gamma=.1$.
The observed drive is constant within each unit sampling interval,
$u(t)=u_j$ for $t\in[j,j+1)$, with
\begin{equation}
u_{j+1}=.9u_j+\sqrt{1-.9^2}\,\epsilon_{j+1},
\qquad
u_0,\epsilon_1,\epsilon_2,\ldots
\overset{\mathrm{iid}}{\sim}\mathcal N(0,1).
\label{eq:probe-mg-drive}
\end{equation}
Thus, $\lambda_D$ controls the amplitude of the additive, unit-variance
AR(1) drive; $\lambda_D=0$ recovers the autonomous system.
The future drive is supplied to the probes as observed input.

\begin{table}[h!]
\centering
\small
\setlength{\tabcolsep}{4pt}
\caption{\textbf{Term magnitudes and probe scores across drive strengths in Driven MG.}
Let $A(t)=\beta x(t-\tau)/(1+x(t-\tau)^{10})-\gamma x(t)$
and $D(t)=\lambda_D u(t)$.
RMS values pool unnormalized training trajectories, time points, and
channels after excluding the first delay window.
Probe scores are long-horizon averages over $k\geq 8$, reported
separately for the linear and nonlinear readouts before averaging
or clipping at zero.}
\vspace{2pt}
\label{tab:driven-mg-term-scales}
\label{tab:probe-validation}
\begin{tabular}{@{}crrrrrrr@{}}
\toprule
& \multicolumn{3}{c}{Term magnitudes}
& \multicolumn{2}{c}{Linear}
& \multicolumn{2}{c}{Nonlinear} \\
\cmidrule(lr){2-4}
\cmidrule(lr){5-6}
\cmidrule(lr){7-8}
$\lambda_D$
& $\mathrm{RMS}(A)$
& $\mathrm{RMS}(D)$
& $\frac{\mathrm{RMS}(D)}{\mathrm{RMS}(A)}$
& $\Psi_D$ & $\Delta\Psi_H$
& $\Psi_D$ & $\Delta\Psi_H$ \\
\midrule
0.03150 & 0.0420 & 0.0315 & 0.75 & 0.023 & 0.802 & 0.040 & 0.522 \\
0.05925 & 0.0591 & 0.0592 & 1.00 & 0.144 & 0.757 & 0.135 & 0.533 \\
0.08000 & 0.0724 & 0.0799 & 1.10 & 0.339 & 0.680 & 0.283 & 0.463 \\
0.13000 & 0.1039 & 0.1299 & 1.25 & 0.645 & 0.568 & 0.504 & 0.362 \\
0.34000 & 0.2424 & 0.3398 & 1.40 & 0.930 & 0.598 & 0.710 & 0.086 \\
\bottomrule
\end{tabular}
\end{table}

We hold the
autonomous dynamics, underlying drive realizations, and data splits fixed, and
vary only the drive scale $\lambda_D$, showing settings from $.03$ to $.34$.  The corresponding magnitudes of autonomous and driven terms are shown in Table~\ref{tab:driven-mg-term-scales}. At each setting, we
compute the drive-only gain $\Psi_D$ and additional history gain
$\Delta\Psi_H$ defined above, using long-horizon averages over $k\geq 8$.
Table~\ref{tab:probe-validation} also reports the linear and nonlinear
probe scores separately, before averaging or clipping.
For the plotted coordinates in Figure~\ref{fig:probe-validation}, the linear and nonlinear probe estimates are first averaged, after which each
coordinate is clipped at zero, following the construction used in the main
predictive-regime figure.

\begin{figure}[htbp!]
  \centering
  \includegraphics[width=0.55\linewidth]{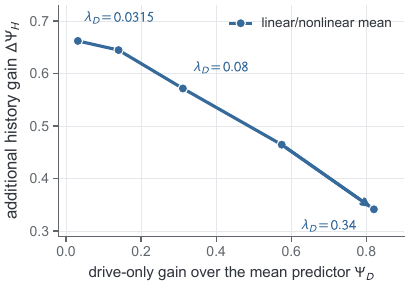}
  \setlength{\abovecaptionskip}{2pt}
  \caption[Probe response to controlled drive strength.]{
  \textbf{Probe response to controlled drive strength.}
  Each point shows the predictive-regime coordinates obtained at one
  Mackey--Glass drive scale $\lambda_D$; the line connects settings in order of
  increasing drive strength, over $0.03150\leq\lambda_D\leq0.34$.
  The autonomous dynamics, drive realizations, and data splits are held fixed.}
  \label{fig:probe-validation}
\end{figure}

Figure~\ref{fig:probe-validation} shows a rightward movement over the
displayed range: as $\lambda_D$ increases from $.03$ to $.34$, the measured
drive-only gain $\Psi_D$ increases from approximately $.03$ to $.82$.
Meanwhile, $\Delta\Psi_H$ decreases while remaining positive, suggesting that the additional predictive value
of history beyond the observed drive weakens.
Thus, the probes respond in the intended direction to increasing observed
drive while still detecting additional history value.

\section{Controlled Test of Strong Observed Drive}
\label{app:driven-mg-boundary}

We isolate the effect of observed drive by comparing the autonomous
Mackey--Glass testbed with a variant that retains the same delay dynamics and
adds a unit-variance AR(1) drive with coefficient .9 and scale .08. For each of the eight independent channels, the observed forcing is
an independent stationary AR(1) process with coefficient .9 and unit variance,
generated using Gaussian innovations.  It enters the Mackey--Glass dynamics
additively with scale .08 and is held constant within each unit sampling
interval.  We retain the autonomous parameters {\(\tau=30\), \(\beta=0.2\), \(\gamma=0.1\)}, and
exponent 10, and integrate the dynamics using RK4 (fourth-order Runge–Kutta method) with step size .1 after a
1000-step burn-in.  We generate 40 trajectories of length 2048 and split them
into 28/6/6 training, validation, and test trajectories.  The eight future
forcing channels are provided to the forecaster as observed future drive. The
driven variant continues to have substantial long-horizon history value, but
now also has a strong observed drive.  Both conditions use \(K=32\), the same
matched training protocol, and dense \(1{:}48\) relative-\(L_2\) evaluation.

\begin{table}[h!]
\centering
\caption{\textbf{Increasing observed drive reverses the effect of
Internal-DW on Mackey--Glass.}  Values are mean\(\pm\)standard deviation over
three matched seeds; lower is better.  Relative \(L_2\) is normalized within
each condition, so comparisons are made within rows.}
\label{tab:driven-mg-boundary}
\small
\setlength{\tabcolsep}{5pt}
\begin{tabular}{lccc}
\toprule
condition & Full BPTT & Internal-DW & DW vs. Full \\
\midrule
weak drive (autonomous) & \(1.05412\pm.00436\) & \(\mathbf{.99966\pm.00958}\) & 5.17\% lower \\
strong observed drive & \(\mathbf{.73489\pm.01064}\) & \(.79157\pm.00350\) & 7.71\% higher \\
\bottomrule
\end{tabular}
\end{table}

{As shown in Table~\ref{tab:driven-mg-boundary}, }Internal-DW improves the autonomous condition but is worse after
the observed drive is strengthened. This controlled comparison shows that
substantial history alone does not guarantee improvement and provides a
concrete stress test of the current estimator's drive-related boundary.

\section{How Forecasting Error Is Evaluated}
\label{app:evaluation-protocol}

Let \(\mathcal U\) denote the held-out test units and \(\mathcal O_u\) the
forecast origins selected within unit \(u\).  We roll the model out once from
each origin without resets through \(H_{\rm eval}=\lceil1.5K\rceil\) and compute
\Eqref{eq:dense-multistart-rel-l2}.
We use at most 64 deterministic evenly spaced origins per test unit, remove
duplicate absolute origins, average origins within a unit, and then weight
units equally.  The sole headline score averages relative \(L_2\) over
\(1{:}H_{\rm eval}\); trained-window, post-window, and terminal summaries are
not reported as separate performance criteria.

\section{How Per-Horizon Gradient Utility Is Measured}
\label{app:result-diagnostics}
\label{app:heldout-gradient-utility}
This appendix gives the protocol behind
Fig.~\ref{fig:application-data-diagnostics}a. At a fixed seed-0 full BPTT
checkpoint, we pair a calibration example $A$ from the training split with a
disjoint example $B$ from the test split, using at most one
common-relative-position rollout start from each loader item. For every
forecast step $k$, we compute the parameter gradient $g_k^A$ produced by that
step's loss and the full-horizon held-out BPTT gradient $\bar g^B$. All horizons
and pairs within a dataset use one fixed sample of parameter coordinates with
the corresponding full-dimensional scaling. We summarize each horizon by its
relative gradient norm and held-out utility:
\begin{equation}
A(k):=\frac{\lVert g_k^A\rVert_2}{\lVert g_1^A\rVert_2},
\qquad
U(k):=
\left\langle
\frac{g_k^A}{\lVert g_k^A\rVert_2},
\bar g^B
\right\rangle,
\qquad
\bar g^B:=\frac{1}{K}\sum_{j=1}^{K}g_j^B .
\label{eq:real-per-horizon-gradient-utility}
\end{equation}

Here, $A(k)$ measures relative gradient magnitude, while $U(k)=\langle g_k^A/\|g_k^A\|,\bar g^B\rangle$ measures
alignment with the held-out full-horizon gradient. Under the infinitesimal update
\(-\epsilon g_k^A/\|g_k^A\|\), the held-out rollout loss changes by
\(-\epsilon U(k)+O(\epsilon^2)\); hence \(U(k)<0\) means that the training
direction raises held-out loss.  The black curve and band in the figure are the
median and 10--90\% range of \(A(k)\), while the purple curve and band are the
median and interquartile range of \(U(k)\).  The displayed percentage is the
fraction of forecast steps whose median \(U(k)\) is negative.
The numbers of disjoint train--test pairs are 6 for Mackey--Glass and
NARMA-5, 8 for ETTm1 and ETTm2, {3 per single-recording iEEG participant and 6 per two-recording participant}, 4 for movie fMRI and shear flow,
and 8 for WeatherBench-2.
{The iEEG probe uses seed 0 for each of 16 participants. Within-participant medians are aggregated by taking the median across participants at each horizon. Its gray and purple bands are respectively the 10th--90th and 25th--75th percentiles across participant curves. The {70\% negative-horizon annotation (45 of 64 forecast steps, rounded to the nearest whole percent)} is computed from this displayed median, not by averaging participant-specific negative fractions.}

\section{Dense-Horizon Relative-$L_2$ Results}
\label{app:quantitative-table}

Table~\ref{tab:absolute-relative-l2} reports the dense-horizon test relative
$L_2$ values underlying Fig.~\ref{fig:assigned-estimator-performance}.
The values are computed using the same evaluation horizons and forecast origins
and are reported as mean $\pm$ standard deviation over three matched seeds;
lower is better. Because the normalization is dataset-specific, values are intended for within-dataset rather than across-dataset comparison.

\begin{table*}[t]
\centering
\caption{\textbf{Dense-horizon relative-$L_2$ values underlying
Fig.~\ref{fig:assigned-estimator-performance}.} Each entry reports the
mean $\pm$ sample standard deviation over three matched seeds. Lower is better.}
\label{tab:absolute-relative-l2}

\footnotesize
\setlength{\tabcolsep}{2pt}
\renewcommand{\arraystretch}{1.12}

\textit{History-dominated, weak drive}

\smallskip

\begin{tabular*}{\textwidth}{@{\extracolsep{\fill}}lcccc@{}}
\toprule
Method & MG & ETTm1 & ETTm2 & Shear flow \\
\midrule
Full BPTT
& {$1.0541\pm0.0044$}
& {$0.9422\pm0.0252$}
& {$0.8684\pm0.0913$}
& {$0.1639\pm0.0354$} \\

Clip
& {$1.0563\pm0.0040$}
& {$0.9497\pm0.0348$}
& {$0.8817\pm0.0771$}
& {$0.1383\pm0.0060$} \\

JReg
& {$1.0556\pm0.0069$}
& {$0.9274\pm0.0293$}
& {$0.8700\pm0.0910$}
& {$0.1397\pm0.0021$} \\

TBPTT
& {$1.0344\pm0.0172$}
& {$0.9340\pm0.0508$}
& {$0.8650\pm0.0552$}
& {$\boldsymbol{0.1322\pm0.0046}$} \\

Static
& {$1.0047\pm0.0073$}
& {$0.9250\pm0.0113$}
& {$0.8043\pm0.0686$}
& {$0.2330\pm0.0057$} \\

Internal-DW
& {$\boldsymbol{0.9997\pm0.0096}$}
& {$\boldsymbol{0.8786\pm0.0176}$}
& {$\boldsymbol{0.7489\pm0.0489}$}
& {$0.1452\pm0.0247$} \\
\bottomrule
\end{tabular*}

\vspace{0.7em}

\textit{Identified boundaries}

\smallskip

\begin{tabular*}{\textwidth}{@{\extracolsep{\fill}}lcccc@{}}
\toprule
Method & NARMA-5 & iEEG & Movie fMRI & WeatherBench-2 \\
\midrule
Full BPTT
& {$0.8707\pm0.0150$}
& {{$1.8871\pm0.0309$}}
& {$\boldsymbol{0.9221\pm0.0109}$}
& {$0.6111\pm0.0005$} \\

Clip
& {$\boldsymbol{0.8068\pm0.0024}$}
& {$1.8846\pm0.0291$}
& {$0.9257\pm0.0141$}
& {$0.5994\pm0.0024$} \\

JReg
& {$0.8731\pm0.0115$}
& {$1.8875\pm0.0312$}
& {$0.9256\pm0.0135$}
& {$\boldsymbol{0.5962\pm0.0010}$} \\

Internal-DW
& {$0.9391\pm0.0138$}
& {{$\boldsymbol{1.8740\pm0.0316}$}}
& {$0.9842\pm0.0211$}
& {$0.6117\pm0.0029$} \\
\bottomrule
\end{tabular*}
\end{table*}

\paragraph{Matched-seed uncertainty.}
For each dataset in Table~\ref{tab:absolute-relative-l2}, we form $d_s=L_{\mathrm{DW},s}-L_{\mathrm{Full},s}$ for matched training seeds $s\in\{0,1,2\}$. We report $\bar d$ and the pointwise 95\% paired Student-$t$ interval $\bar d\pm t_{0.975,2}s_d/\sqrt{3}$, where $s_d$ is the sample standard deviation of the three paired differences. Negative differences favor Internal-DW. For iEEG, each seed-level error first averages the same 16 participants equally.
\begin{center}
\footnotesize
\setlength{\tabcolsep}{4pt}
\begin{tabular}{lrr}
\toprule
Dataset & $\bar d$ & Pointwise 95\% CI \\
\midrule
MG & $-0.05446$ & $[-0.08212, -0.02679]$ \\
ETTm1 & $-0.06357$ & $[-0.14635, +0.01922]$ \\
ETTm2 & $-0.11948$ & $[-0.33903, +0.10007]$ \\
Shear flow & $-0.01863$ & $[-0.04541, +0.00816]$ \\
NARMA-5 & $+0.06843$ & $[+0.04108, +0.09578]$ \\
iEEG & $-0.01312$ & $[-0.03130, +0.00507]$ \\
Movie fMRI & $+0.06212$ & $[-0.01621, +0.14045]$ \\
WeatherBench-2 & $+0.00057$ & $[-0.00789, +0.00903]$ \\
\bottomrule
\end{tabular}
\end{center}

Internal-DW has a negative mean paired difference on all four history-dominated, weak-drive testbeds. The pointwise 95\% interval lies entirely below zero for MG. Although the intervals for ETTm1, ETTm2, and shear flow include zero, all three seed-level differences are negative in each case, indicating a consistent direction of improvement across the evaluated runs.

\section{End-to-end training time}
\label{app:computation-time}

We measure the wall-clock cost of the assigned full BPTT and
Internal-DW training configurations on an NVIDIA H200 NVL GPU. Within each
matched pair, the two arms use the same dataset, model, training horizon,
batch geometry, optimizer, and physical GPU. We discard one warm-up epoch
(iEEG uses 40 warm-up training batches) and use the median
time of the next three epochs. Table~\ref{tab:training-time} reports the
mean and sample standard deviation of these within-run medians over three
paired repeats; the execution order is alternated across repeats to reduce
order and thermal effects. CUDA is explicitly synchronized at both timing
boundaries. The timer covers the training loop only and therefore excludes
validation and checkpoint I/O.

The Internal-DW measurements include its online covariance probes, routewise
gain solves, and application of the resulting gains. Construction of a
train-only domain-prior artifact, when used, is a one-time preprocessing step
and is excluded. Thus, this is an end-to-end comparison of the training
implementations used in our forecasting experiments rather than an isolated
solver microbenchmark.

\begin{table}[t]
  \centering
  \caption{\textbf{End-to-end training time on one NVIDIA H200 NVL.}
  Each entry is the mean $\pm$ sample standard deviation across three
  paired timing repeats. Within each run, we discard one warm-up epoch
  (iEEG uses 40 warm-up training batches) and take the median time over the next three epochs. The final column
  summarizes the paired percentage increase of Internal-DW over
  Full BPTT.}
  \label{tab:training-time}
  \small
  \begin{tabular}{lrrr}
    \toprule
    Data
    & Full BPTT (s/epoch)
    & Internal-DW (s/epoch)
    & Increase (\%) \\
    \midrule
    Mackey--Glass
    & $21.28 \pm 0.09$
    & $27.43 \pm 0.42$
    & $28.9 \pm 1.7$ \\
    ETTm1
    & $27.78 \pm 0.06$
    & $35.32 \pm 0.23$
    & $27.1 \pm 0.9$ \\
    ETTm2
    & $28.68 \pm 0.08$
    & $36.31 \pm 0.31$
    & $26.6 \pm 1.0$ \\
    Shear flow
    & $29.79 \pm 0.43$
    & $41.55 \pm 0.80$
    & $39.5 \pm 0.7$ \\
    NARMA-5
    & $22.26 \pm 0.13$
    & $28.31 \pm 0.15$
    & $27.2 \pm 0.2$ \\
    iEEG (per subject)
    & $31.64 \pm 0.03$
    & $44.83 \pm 0.39$
    & $41.7 \pm 1.2$ \\
    Movie fMRI
    & $498.67 \pm 1.98$
    & $599.90 \pm 3.04$
    & $20.3 \pm 0.1$ \\
    WeatherBench-2
    & $32.84 \pm 0.39$
    & $43.38 \pm 0.39$
    & $32.1 \pm 0.9$ \\
    \bottomrule
  \end{tabular}
\end{table}

Across the eight testbeds, the mean paired increase in epoch time ranges from
20.3--41.7\%. This additional training cost reflects the online
estimation and routewise gradient scaling that distinguish Internal-DW from
full BPTT. The forward model is unchanged, and Internal-DW introduces no
inference-time computation.

\section{Per-seed fitted curves}
\label{app:per-seed-curve}

Figure~\ref{fig:per-seed-fitted-curves} decomposes the averaged
training-horizon trends in Fig.~\ref{fig:fixed-horizon-k-sweep} into
individual seeds. For each dataset, seed, and method, we independently fit a
least-squares quadratic constrained to have nonnegative curvature. A boundary
solution may reduce to a linear trend, in which case the fitted optimum lies
at an endpoint of the evaluated range. The gray and blue vertical dashed
lines indicate the fitted optima for Full BPTT and Internal-DW, respectively.

The Internal-DW optimum lies at or to the right of the Full-BPTT optimum in
all 12 dataset--seed comparisons. The shift is strict in 11 comparisons,
while ETTm1 seed 0 gives the same boundary optimum for both methods. Thus, the
shift toward a longer fitted optimal training horizon is not produced by averaging
across seeds.

\begin{figure*}[h]
  \centering
  \setlength{\abovecaptionskip}{3pt}
  \includegraphics[width=\textwidth]
  {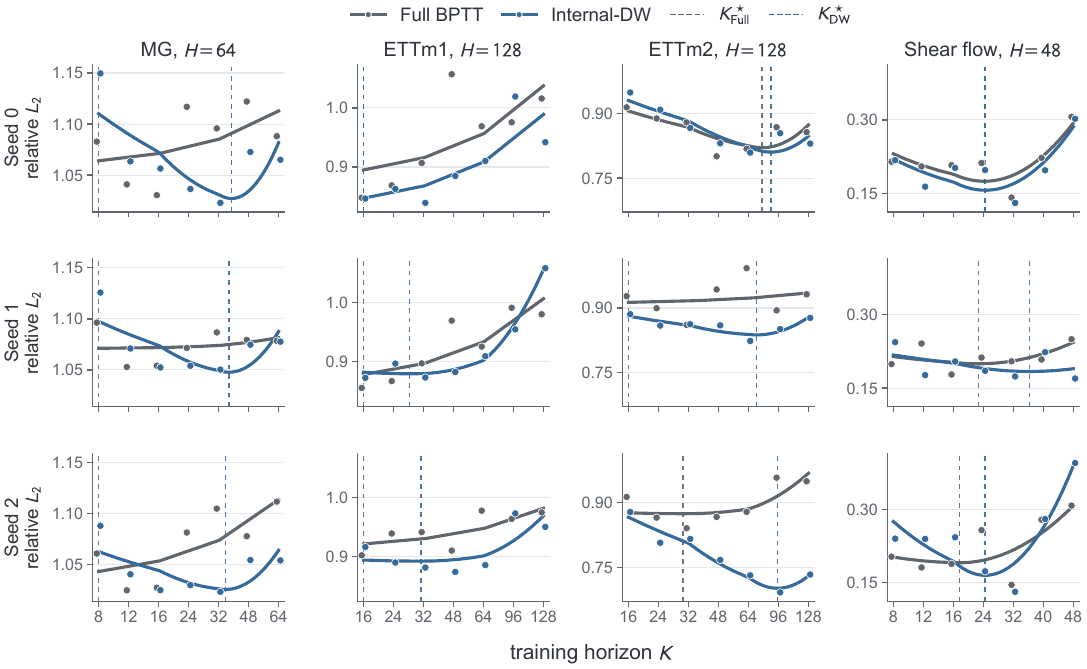}
  \caption{\textbf{Seed-specific fitted training-horizon curves.}
  Columns correspond to datasets and rows to matched seeds. Points show
  relative $L_2$ for checkpoints trained at different horizons $K$ and
  evaluated at the fixed forecast horizon $H$ shown above each column.
  Curves are fitted independently for every seed and method. Gray and blue
  vertical dashed lines mark the fitted optima of Full BPTT and Internal-DW,
  respectively.}
  \label{fig:per-seed-fitted-curves}
\end{figure*}

\end{document}

%% file: math_commands.tex
\usepackage{amsmath,amsfonts,bm}

\def\eqref#1{equation~\ref{#1}}
\def\Eqref#1{Equation~\ref{#1}}

\def\1{\bm{1}}

\DeclareMathAlphabet{\mathsfit}{\encodingdefault}{\sfdefault}{m}{sl}
\SetMathAlphabet{\mathsfit}{bold}{\encodingdefault}{\sfdefault}{bx}{n}

\newcommand{\E}{\mathbb{E}}

